\PassOptionsToPackage{unicode}{hyperref}
\PassOptionsToPackage{hyphens}{url}
\PassOptionsToPackage{dvipsnames,svgnames*,x11names*}{xcolor}
\documentclass[12pt]{article}
\usepackage{lmodern}
\usepackage{amssymb,amsmath}

\DeclareMathOperator*{\argmin}{arg\,min}
\usepackage{ifxetex,ifluatex}
\ifnum 0\ifxetex 1\fi\ifluatex 1\fi=0 % if pdftex
  \usepackage[T1]{fontenc}
  \usepackage[utf8]{inputenc}
  \usepackage{textcomp} % provide euro and other symbols
\else % if luatex or xetex
  \usepackage{unicode-math}
  \defaultfontfeatures{Scale=MatchLowercase}
  \defaultfontfeatures[\rmfamily]{Ligatures=TeX,Scale=1}
\fi
\IfFileExists{upquote.sty}{\usepackage{upquote}}{}
\IfFileExists{microtype.sty}{% use microtype if available
  \usepackage[]{microtype}
  \UseMicrotypeSet[protrusion]{basicmath} % disable protrusion for tt fonts
}{}
\makeatletter
\@ifundefined{KOMAClassName}{% if non-KOMA class
  \IfFileExists{parskip.sty}{%
    \usepackage{parskip}
  }{% else
    \setlength{\parindent}{0pt}
    \setlength{\parskip}{6pt plus 2pt minus 1pt}}
}{% if KOMA class
  \KOMAoptions{parskip=half}}
\makeatother
\usepackage{xcolor}
\IfFileExists{xurl.sty}{\usepackage{xurl}}{} % add URL line breaks if available
\IfFileExists{bookmark.sty}{\usepackage{bookmark}}{\usepackage{hyperref}}
\hypersetup{
  pdftitle={Connecting the Popularity Adjusted Block Model to the Generalized Random Dot Product Graph for Community Detection and Parameter Estimation},
  colorlinks=true,
  linkcolor=Maroon,
  filecolor=Maroon,
  citecolor=Blue,
  urlcolor=blue,
  pdfcreator={LaTeX via pandoc}}
\usepackage{graphicx,grffile}
\def\spacingset#1{\renewcommand{\baselinestretch}%
{#1}\small\normalsize} \spacingset{1}

\makeatletter
\def\maxwidth{\ifdim\Gin@nat@width>\linewidth\linewidth\else\Gin@nat@width\fi}
\def\maxheight{\ifdim\Gin@nat@height>\textheight\textheight\else\Gin@nat@height\fi}
\makeatother
\setkeys{Gin}{width=\maxwidth,height=\maxheight,keepaspectratio}
\makeatletter
\def\fps@figure{htbp}
\makeatother
\providecommand{\tightlist}{%
  \setlength{\itemsep}{0pt}\setlength{\parskip}{0pt}}
\usepackage{float}
\usepackage{mathtools}
\usepackage[backend=bibtex,maxbibnames = 10,style=authoryear,giveninits=true,natbib]{biblatex}
\DeclareNameAlias{sortname}{last-first}
\renewbibmacro{in:}{%
  \ifentrytype{article}{}{%
  \printtext{\bibstring{in}\intitlepunct}}}
\usepackage{parskip}
\usepackage[linesnumbered,ruled,vlined]{algorithm2e}
\usepackage{verbatim}
\usepackage{amsthm}
\usepackage{comment}
\title{Popularity Adjusted Block Models are Generalized Random Dot Product Graphs}

\author{John Koo\\Department of Statistics, Indiana University\\ \\
Minh Tang\\Department of Statistics, North Carolina State University\\ \\
Michael W. Trosset\\Department of Statistics, Indiana University\\ \\
September 8, 2021}

\date{}

\begin{document}

\maketitle

\begin{abstract}
We connect two random graph models, the Popularity Adjusted Block Model
(PABM) and the Generalized Random Dot Product Graph (GRDPG),
by demonstrating that the PABM is a special case of 
the GRDPG in which communities correspond to 
mutually orthogonal subspaces of latent vectors. 
This insight allows us to construct new algorithms for community detection
and parameter estimation for the PABM, 
as well as improve an existing algorithm 
that relies on Sparse Subspace Clustering. 
Using established asymptotic properties of 
Adjacency Spectral Embedding for the GRDPG, 
we derive asymptotic properties of these algorithms. 
In particular, we demonstrate that the absolute number of 
community detection errors tends to zero as 
the number of graph vertices tends to infinity. 
Simulation experiments illustrate these properties. 
\end{abstract}

\noindent%

\providecommand{\tightlist}{%
  \setlength{\itemsep}{0pt}\setlength{\parskip}{0pt}}
\newcommand{\diag}{\mathrm{diag}}
\newcommand{\tr}{\mathrm{Tr}}
\newcommand{\blockdiag}{\mathrm{blockdiag}}
\newcommand{\indep}{\stackrel{\mathrm{ind}}{\sim}}
\newcommand{\iid}{\stackrel{\mathrm{iid}}{\sim}}
\newcommand{\Bernoulli}{\mathrm{Bernoulli}}
\newcommand{\Betadist}{\mathrm{Beta}}
\newcommand{\BG}{\mathrm{BernoulliGraph}}
\newcommand{\PABM}{\mathrm{PABM}}
\newcommand{\RDPG}{\mathrm{RDPG}}
\newcommand{\GRDPG}{\mathrm{GRDPG}}
\newcommand{\Multinomial}{\mathrm{Multinomial}}
\newtheorem{theorem}{Theorem}
\newtheorem{lemma}{Lemma}
\newtheorem{proposition}{Proposition}
\theoremstyle{remark}
\newtheorem{remark}{Remark}
\theoremstyle{definition}
\newtheorem{definition}{Definition}
\newtheorem{example}{Example}
\newpage
\spacingset{1.5} % DON'T change the spacing!

\hypertarget{introduction}{%
\section{Introduction}\label{introduction}}

Statistical inference on random graphs requires a suitable probability
model.  A general probability model for unweighted and undirected
graphs is the Bernoulli Graph (also known as the inhomogeneous
Erd\"{o}s-R\'{e}nyi model), which assumes that edges occur as
independent Bernoulli trials.  A Bernoulli Graph is characterized by
an edge probability matrix $P=[P_{ij}]$, where an edge between
vertices $i$ and $j$ occurs with success probability $P_{ij}$.  A
trivial example of a Bernoulli Graph is the (homogeneous)
Erd\"{o}s-R\'{e}nyi model proposed by \citet{Gilbert:1959}, in which
the vertices of the random graph are fixed and possible edges occur
independently with fixed probability $P_{ij} = p$ for all $i,j$.  The requirement that
$P_{ij} \equiv p$ for all $i,j$ is too strong for most applications, and various
researchers have weakened that requirement in various ways.  The
present work relates two lines of generalization.

Network analysis is often concerned with community detection.  One form of community detection assumes that each vertex belongs to an unobserved community, with the probability of an edge between vertices $i$ and $j$ depending on the communities to which $i$ and $j$ belong.  Formally, one assigns each vertex $v_i$ a community label $z_i$ and assumes a Bernoulli Graph in which $P_{ij}$ is a function of $z_i$ and $z_j$.  Such models, called Block Models, define the goal of community detection as a problem in statistical inference: identify the true community (up to permutation of labels) to which each vertex belongs.

The classical Stochastic Block Model (SBM) of \citet{doi:10.1080/0022250X.1971.9989788} specifies that each edge probability $P_{ij}$ depends only on the labels $z_i$ and $z_j$, i.e., $P_{ij} = \omega_{z_i,z_j}$.  Subsequent researchers have weakened this assumption.   The Degree-Corrected Block Model (DCBM) of \citet{Karrer_2011} assigns an additional parameter $\theta_i$ to each vertex and sets $P_{ij} = \theta_i \theta_j \omega_{z_i,z_j}$. The Popularity Adjusted Block Model (PABM) of \citet{307cbeb9b1be48299388437423d94bf1} generalizes the DCBM, allowing heterogeneity of edge probabilities within and between communities while still maintaining distinct community structure.

Another type of Bernoulli Graph was proposed by
\citet*{10.1007/978-3-540-77004-6_11}.  A Random Dot Product Graph
(RDPG) specifies that each vertex corresponds to a latent position
vector in Euclidean space and that the probability of an edge between
two vertices is the dot product of their latent position vectors.
Thus, if the latent positions are $x_1,\ldots,x_n \in \mathbb{R}^d$
and $X = \bigl[ x_1 \mid \dots \mid x_n \bigr]^\top$,
then the edge probability matrix is $P = XX^\top$.  Clearly, any
Bernoulli Graph with positive semidefinite $P$ is an RDPG.  The
positive definite Euclidean inner product in the RDPG model was
replaced by an indefinite inner product in
\citet{rubindelanchy2017statistical}, resulting in the {\em
Generalized}\/ RDPG (GRDPG).

In contrast to Block Models, neither RDPGs nor GRDPGs inherently
specify distinct communities.  However, one can easily impose
community structure by assuming that the latent positions lie in
distinct clusters.  Hence, it is not surprising that Block Models can
be studied by reformulating them as RDPGs or GRDPGs.  For example, an
assortative SBM (an SBM for which $P$ is positive semidefinite) is
equivalent to an RDPG for which all vertices in the same community
correspond to the same latent position vector.
Likewise, the DCBM is equivalent to an RDPG for which all vertices in
the same community correspond to latent position vectors that lie on a
straight line. 

Because the edge probability matrix of a PABM is not necessarily
positive semidefinite, a PABM is not necessarily an RDPG.  In Section
\ref{connecting-the-pabm-to-the-grdpg} we demonstrate that every PABM
is in fact a specific type of GRDPG for which the latent position
vectors lie in distinct orthogonal subspaces, each subspace
corresponding to a community.  This identification is our central
result.  In Section \ref{methods}, we use the geometry of the GRDPG to
derive more efficient algorithms for detecting the communities and
estimating the parameters in the PABM.  We report the results of
simulation studies in Section \ref{simulated-examples} and apply our
methods to three well-known data sets in Section
\ref{real-data-examples}. Section \ref{discussion} concludes.
Proofs of Theorems 1 and 2 are provided in the body of the text while proofs of 
Theorems 3, 4, and 5 are provided in the Appendix. 

\hypertarget{connecting-the-popularity-adjusted-block-model-to-the-generalized-random-dot-product-graph}{%
\section{PABMs are GRDPGs}\label{connecting-the-popularity-adjusted-block-model-to-the-generalized-random-dot-product-graph}}

In this section, we show that the PABM is a special case of the GRDPG. 
More specifically, a graph $G$ drawn from the PABM can be represented by 
a collection of latent vectors in Euclidean space. 
We further show that the latent configuration that
induces the PABM consists of orthogonal subspaces with each subspace
corresponding to a community.

\hypertarget{notation}{%
\subsection{Notation and Scope}\label{notation}}

Let $G = (V, E)$ be an unweighted, undirected graph without self-loops, with vertex set $V$ ($|V| = n$) and edge set $E$. The matrix
$A \in \{0, 1\}^{n \times n}$ represents the adjacency matrix of $G$ 
such that $A_{ij} = 1$ if there exists an edge between vertices $i$ and $j$ 
and $0$ otherwise. 
$A_{ij} = A_{ji}$ and $A_{ii} = 0$ for each $i, j \in [n]$ (where $[n] = \{1, 2, ..., n\}$). 
We further restrict our analyses to Bernoulli graphs. 
Let $P \in [0, 1]^{n \times n}$ be a symmetric matrix of edge probabilities. 
Graph $G$ is sampled from $P$ by drawing $A_{ij} \indep \Bernoulli(P_{ij})$ 
for each $1 \leq i < j \leq n$ (setting $A_{ji} = A_{ij}$ and $A_{ii} = 0$). 
We denote $A \sim \BG(P)$ as a graph with adjacency matrix $A$ 
sampled from edge probability matrix $P$ in this manner. 
If each vertex has a hidden label in $[K]$, 
they are denoted as $z_1, ..., z_n$. 
$\lambda_{ik}$ denotes the popularity parameter of vertex $i$ to community $k$. 
$\Lambda$ is the $n \times K$ matrix of popularity parameters. 
Finally, we denote 
$X = \bigl[x_1 \mid \cdots \mid x_n\bigr]^\top \in \mathbb{R}^{n \times d}$ 
as the matrix corresponding to a collection of $n$ latent vectors $x_1, ..., x_n \in \mathbb{R}^d$.

\hypertarget{the-popularity-adjusted-block-model-and-the-generalized-random-dot-product-graph}{%
\subsection{Two Probability Models for Graphs}\label{the-popularity-adjusted-block-model-and-the-generalized-random-dot-product-graph}}

\begin{definition}[Popularity Adjusted Block Model]
\label{pabm}
Let $K \geq 1$ be an integer and let $\Lambda \in \mathbb{R}^{n \times
  K}$ be a matrix with entries in $[0,1]$. Let $z_1, z_2, \dots, z_n
\in [K]$. A graph $G$ with adjacency matrix $A$ is said to
be a popularity adjusted block model graph with $K$ communities, popularity vectors
$\Lambda$, and sparsity parameter $\rho_n \in (0,1]$ if $A \sim \BG(P)$
where the edge probability matrix $P$ has entries $P_{ij}$, $i, j \in [n]$, of the form
$$P_{ij} = \rho_n \lambda_{iz_j} \lambda_{jz_i}.$$
\end{definition}

\begin{remark}
  \label{rem:pabm_view2}
In a PABM, each vertex $i$ has $K$ popularity parameters 
$\lambda_{i1}, \dots, \lambda_{iK}$, that describe its affinity
toward each of the $K$ communities. 
Another view of a PABM is as follows.
Let $\tilde{P}$ be the matrix obtained by permuting the rows and
columns of $P$ so that the vertices are reorganized by community memberships $z_i \in \{1,2,\dots,K\}$ in increasing order. 
Denote by $\tilde{P}^{(k \ell)}$ the $n_k \times n_{\ell}$
submatrix of $\tilde{P}$ corresponding to the
edge probabilities between vertices in communities
$k$ and $\ell$. Here $n_k = |\{ i \colon z_i = k\}|$ is the number of
vertices assigned to community $k$, for $k = 1,2\dots,K$.
Note that $\tilde{P}^{(k \ell)} = (\tilde{P}^{(\ell k)})^\top$. Next
let $\lambda^{(k \ell)} = \{\lambda_{i \ell} \colon z_i = k\} \in
\mathbb{R}^{n_k}$; the elements of $\lambda^{(k \ell)}$ are the affinity parameters toward the $\ell$th community of
all vertices in the $k^{th}$ community. Define $\lambda^{(\ell k)}$
analogously. Then each block $\tilde{P}^{(k \ell)}$ can be written as the outer product of two vectors:
\begin{equation} \label{eq:pabm}
  \tilde{P}^{(k \ell)} = \rho_n \lambda^{(k \ell)} (\lambda^{(\ell k)})^{\top}.
\end{equation} 

We will henceforth use the notation \(A \sim \PABM(\{\lambda^{(k
  \ell)}\}_K, \rho_n)\) to denote
a random adjacency matrix \(A\) drawn from a PABM with $K$ communities, popularity parameters
\(\{\lambda^{(k \ell)}\}\) and sparsity parameter $\rho_n$.
\end{remark}

The sparsity parameter $\rho_n$ in the definition of the PABM
influences the degrees of the vertices in the sampled graphs $A \sim
\BG(P)$. In particular, for a fixed $\Lambda$, the graphs become
sparser as $\rho_n$ decreases. Note that $\rho_n$ and $\Lambda$ are not
uniquely identifiable, i.e., we can scale $\rho_n$ by a constant $c > 0$ and
scale $\Lambda$ by $c^{-1/2}$ without changing the edge probabilities
in $P$. Thus, for ease of exposition, we shall assume henceforth that
(1) $\Lambda$ is normalized to have Frobenius norm $\|\Lambda\|_{F} =
\sqrt{n}$ and (2) the $\ell_2$ norms of the rows of
$\|\Lambda\|$ are all bounded away from $0$. Under these two
assumptions, the sparsity parameter $\rho_n$ can now be viewed as
controlling the density of $A$, i.e., the average degree of
$A$ grows at rate $n \rho_n$ and the number of edges grows at rate
$n^2 \rho_n$. The choice $\rho_n \equiv 1$ and $\rho_n \rightarrow 0$
then corresponds to the dense graphs regime and semi-sparse graphs
regime, respectively. 

\begin{definition}[Generalized Random Dot Product Graph]
\label{grdpg}
Let $p \geq 1$ and $q \geq 0$ be integers. Define $I_{p,q}$ as the
block diagonal matrix $I_{p,q} = \Bigl[\begin{smallmatrix} I_{p} & 0
  \\ 0 & - I_{q} \end{smallmatrix} \Bigr]$ where $I_{p}$ and $I_{q}$
are the identity matrices of dimensions $p \times p$ and $q \times q$
respectively. Denote $d = p + q$ and let $\mathcal{X}$ be a subset of $\mathbb{R}^d$ such
that, for any $x \in \mathcal{X}$ and $y \in \mathcal{X}$, we have
$x^{\top} I_{p,q} y \in [0,1]$. 
% Next define $\mathcal{X} = \{x, y \in \mathbb{R}^d : x^\top I_{p,q} y \in [0, 1]\}$ where $d = p + q$. 
Let $X = \bigl[x_1 \mid \cdots \mid x_n\bigr]^{\top}$ be
a $n \times d$ matrix with rows $x_i \in \mathcal{X}$. A graph $G$
with adjacency matrix $A$ is said to be a generalized random dot
product graph with latent positions $X$, sparsity parameter $\rho_n
\in (0,1]$
and signature $(p,q)$ if $A \sim \BG(P)$ where the edge probability
matrix $P$ is given
by $P = \rho_n X I_{p,q} X^{\top}$, i.e., the entries of $P$ are of
the form $P_{ij} = \rho_n x_i^{\top} I_{p,q} x_j.$
\end{definition}
We will use the notation \(A \sim \GRDPG_{p,q}(X;\rho_n)\) to denote a random
adjacency matrix \(A\) drawn from latent positions \(X\), sparsity
parameter $\rho_n$ and signature
\((p, q)\).

\begin{definition}[Indefinite Orthogonal Group]
  \label{def:indefinite}
The indefinite orthogonal group with signature $(p, q)$ is
the set $\{Q \in \mathbb{R}^{d \times d} \colon Q I_{p, q} Q^{\top} = I_{p, q}\}$,
denoted as $\mathbb{O}(p, q)$. Here $d = p + q$. 
\end{definition}

\begin{remark}
  \label{rem:non_identifiable}
The latent vectors that produce $X I_{p,q} X^\top = P$ are not unique
\citep{rubindelanchy2017statistical}.
More specifically, if $P_{ij} = x_i^\top I_{p, q} x_j$, 
then for any $Q \in \mathbb{O}(p, q)$ we also have 
$(Q x_i)^\top I_{p, q} (Q x_j) = x_i^\top (Q^\top I_{p, q} Q) x_j =
x_i^\top I_{p, q} x_j = P_{ij}$. Unlike in the RDPG case, transforming the latent positions via multiplication
by $Q \in \mathbb{O}(p, q)$ does not necessarily maintain interpoint angles or
distances.
\end{remark}

\begin{remark}
  \label{rem:ase}
We can use Adjacency Spectral Embedding (ASE) 
\citep{doi:10.1080/01621459.2012.699795} to recover the latent vectors of a GRDPG. 
This procedure consists of taking the spectral decomposition of $A$
and keeping the $p + q$ largest eigenvalues (in modulus) and
corresponding eigenvectors of $A$. More specifically, let $A$ be an $n
\times n$ adjacency matrix and denote the
eigendecomposition of $A$ as
$$A = \sum_{i} \hat{\lambda}_i \hat{v}_i \hat{v}_i^{\top}, \quad
\text{where} \,\,
|\hat{\lambda}|_1 \geq |\hat{\lambda}_2| \geq \dots \geq
|\hat{\lambda}_n|.$$
Let $\hat{D}$ be a diagonal matrix whose diagonal entries are
the eigenvalues $\{\hat{\lambda}_i\}_{i=1}^{d}$ arranged in decreasing values (not in
decreasing modulus) and let $\hat{V}$ be the $n \times d$ matrix whose
columns are the corresponding eigenvectors $\{\hat{v}_i\}_{i=1}^{d}$
arranged in the same order as the diagonal entries of $\hat{D}$. Now
define $\hat{Z} = \hat{V} |\hat{D}|^{1/2}$ where the $|\cdot|$
operation is applied elementwise to the entries of $\hat{D}$. Then
$\hat{Z}$ serves as an estimate of $X$, up to some non-identifiable
orthogonal transformation $Q$ as described in
Remark~\ref{rem:non_identifiable} and Definition~\ref{def:indefinite};
see \citet{rubindelanchy2017statistical} for further details.
% choosing the $p$ most positive and $q$ most negative
% eigenvalues, $D = \diag(\lambda_1, ..., \lambda_p, \lambda_{n - q +
%   1}, ..., \lambda_n)$, and their corresponding eigenvectors, $\hat{V}
% \in \mathbb{R}^{n \times (p + q)}$, to construct the embedding
%$\hat{Z} = V |D|^{1/2}$. More specifically, 
\end{remark}

\hypertarget{connecting-the-pabm-to-the-grdpg}{%
\subsection{The Geometry of PABMs}\label{connecting-the-pabm-to-the-grdpg}}
Now that we defined the PABM and GRDPG, 
we show the special geometry of the PABM when viewed as a GRDPG. For
ease of exposition, and without loss of generality, we drop the
dependency on the sparsity parameter $\rho_n$ 
and assume $\rho_n \equiv 1$ throughout this subsection. 
\begin{theorem}[The latent configuration of the PABM]
\label{theorem2}
Let $A \sim \mathrm{PABM}(\{\lambda^{(k \ell)}\}_K)$ be an instance of a
PABM with $K \geq 1$ blocks and latent vectors $\{\lambda^{(k \ell)}
\colon 1 \leq k \leq K, 1 \leq \ell \leq K\}$. 
Then there exists a block diagonal matrix
$X \in \mathbb{R}^{n \times K^2}$ defined by $\{\lambda^{(k \ell)}\}$ and a 
$K^2 \times K^2$ {\em fixed} orthonormal matrix $U$ such 
that $A \sim \mathrm{GRDPG}_{K (K+1) / 2, K (K-1) /
  2}(\tilde{\Pi}XU)$. Here $\tilde{\Pi}$ is the permutation matrix
such that $P = \tilde{\Pi} \tilde{P} \tilde{\Pi}^{\top}$ where the
rows and columns of $\tilde{P}$ are arranged according to increasing values of the
community labels (see Remark~\ref{rem:pabm_view2}). 
\end{theorem}

\begin{proof}
We will prove this theorem in two parts. First, for demonstration
purposes, we focus on the case when $K = 2$ to build intuition. 
The general case of $K \geq 2$ is presented later.  

For the $K = 2$ case, the proof is straightforward. We will first work with
the matrix $\tilde{P}$. Note that $\tilde{P}$ has the form

$$\tilde{P} = \begin{bmatrix} P^{(11)} & P^{(12)} \\ P^{(21)} &
  P^{(22)} \end{bmatrix} = \begin{bmatrix} \lambda^{(11)} (\lambda^{(11)})^\top & \lambda^{(12)} (\lambda^{(21)})^\top \\
  \lambda^{(21)} (\lambda^{(12)})^\top & \lambda^{(22)}
  (\lambda^{(22)})^\top \end{bmatrix}.$$
  Now let
$$X = \begin{bmatrix}
\lambda^{(11)} & \lambda^{(12)} & 0 & 0 \\
0 & 0 & \lambda^{(21)} & \lambda^{(22)}
\end{bmatrix} \quad \text{and} \quad
U = \begin{bmatrix} 1 & 0 & 0 & 0 \\
0 & 0 & 1 / \sqrt{2} & 1 / \sqrt{2} \\
0 & 0 & 1 / \sqrt{2} & - 1 / \sqrt{2} \\
0 & 1 & 0 & 0 \end{bmatrix}.$$
Then by straightforward matrix multiplication, we obtain 
$$X U I_{3, 1} U^\top X^\top =
\begin{bmatrix}
  \lambda^{(11)} (\lambda^{(11)})^\top & \lambda^{(12)} (\lambda^{(21)})^\top \\
  \lambda^{(21)} (\lambda^{(12)})^\top & \lambda^{(22)} (\lambda^{(22)})^\top
\end{bmatrix} = \tilde{P}$$
and hence $\tilde{P}$ also corresponds to the edge probability matrix of GRDPG
with latent vectors described by $X U$. As $P = \tilde{\Pi} \tilde{P}
\tilde{\Pi}^{\top}$ we conclude that $P$ has latent vectors described
by $\tilde{\Pi} X U$. 

It is nevertheless instructive to look at a few intermediate steps. 
More specifically, the product $U I_{3, 1} U^\top$ 
yields a permutation matrix $\Pi$ with fixed points at positions $1$ and $4$ 
and a cycle of order 2 swapping positions $2$ and $3$, i.e., 
$$\Pi = U I_{3, 1} U^\top = \begin{bmatrix} 1 & 0 & 0 & 0 \\
  0 & 0 & 1 & 0 \\
  0 & 1 & 0 & 0 \\
  0 & 0 & 0 & 1
\end{bmatrix}.$$
Furthermore, as $U$ is orthonormal and $I_{3, 1}$ is diagonal, 
$U I_{3, 1} U^\top$ is also an eigendecomposition of $\Pi$ where the fixed
points of $\Pi$ are mapped to the eigenvectors $e_1$ and $e_4$
while the cycles of order two are mapped to the eigenvectors  
$\tfrac{1}{\sqrt{2}}(e_{2} + e_3)$ and $\tfrac{1}{\sqrt{2}}(e_{2} -
e_3)$; here $e_i$ denote the $i^\mathrm{th}$ basis vector in $\mathbb{R}^{4}$. 
Thus, another way of decomposing the edge probability matrix is
$\tilde{P} = X \Pi X^\top$ where the rows of $X$ lie in the union of
two 2-dimensional orthogonal subspaces and $\Pi$ is a permutation matrix. 

For the general case, we can extend $\tilde{P} = X \Pi X^\top$ to larger $K$. 
For a more concrete example of this, refer to Example~1. 
We once again consider $\tilde{P}$ as defined in
Remark~\ref{rem:pabm_view2}.  
We first define the following matrices

\begin{gather}
\label{eq:xy}
\Lambda^{(k)} = \begin{bmatrix} \lambda^{(k1)} \mid \cdots \mid \lambda^{(kK)} \end{bmatrix}
\in \mathbb{R}^{n_k \times K}, \quad
X = \mathrm{blockdiag}(\Lambda^{(1)}, \dots, \Lambda^{(K)}) \in
\mathbb{R}^{n \times K^2}, \\
L^{(k)} = \mathrm{blockdiag}(\lambda^{(1k)}, \dots, \lambda^{(Kk)}) \in
\mathbb{R}^{n \times K}, \quad
Y = \begin{bmatrix} L^{(1)} \mid \cdots \mid L^{(K)} \end{bmatrix} \in
\mathbb{R}^{n \times K^2}.
\end{gather}

It is then straightforward to verify that

\begin{gather*}
  XY^{\top} = \mathrm{blockdiag}(\Lambda^{(1)}, \dots,
\Lambda^{(K)}) \begin{bmatrix} L_1^{\top} \\ \vdots \\
  L_{K}^{\top} \end{bmatrix} = \begin{bmatrix} \Lambda^{(1)}
  (L^{(1)})^{\top} \\ \vdots \\
  \Lambda^{(K)} (L^{(K)})^{\top} \end{bmatrix}, \\
\Lambda^{(k)} (L^{(k)})^{\top} = \begin{bmatrix} \lambda^{(k1)}
  (\lambda^{(1k)})^{\top} \mid \dots \mid \lambda^{(kK)}
  (\lambda^{(Kk)})^{\top} \end{bmatrix} = \begin{bmatrix} P^{(k1)}
  \mid P^{(k2)} \mid \dots \mid P^{(kK)} \end{bmatrix}.
\end{gather*}
We therefore have $\tilde{P} = X Y^\top$. 
Similar to the $K = 2$ case, we also have $Y = X \Pi$ for some permutation matrix
$\Pi$ and hence $\tilde{P} = X \Pi X^\top$.
The permutation described by $\Pi$ now has $K$ fixed points, which correspond to
$K$ eigenvalues equal to $1$ with corresponding eigenvectors $e_k$ where
$k = r (K + 1) + 1$ for $0 \leq r \leq K - 1$. It also has
$\tbinom{K}{2}$ cycles of order $2$. Each cycle corresponds to
a pair of eigenvalues $\{-1,+1\}$ and a pair of eigenvectors
$\{(e_s + e_t)/\sqrt{2},(e_s - e_t)/ \sqrt{2}\}$. 

Let $p = K (K + 1) / 2$ and $q = K (K - 1) / 2$. 
We therefore have
\begin{equation} \label{eq:permutation}
\Pi = U I_{p,q} U^\top
\end{equation}
where $U$ is a $K^2 \times K^2$ orthogonal matrix and hence
\begin{equation} \label{eq:pabm-grdpg}
\tilde{P} = X U I_{p, q} (X U)^\top.
\end{equation}
In summary we can describe the PABM with $K$ communities as a GRDPG with latent
positions $\tilde{\Pi} X U$ and known signature $(p,q) = \bigl( \tfrac{1}{2} K (K + 1) ,
\tfrac{1}{2} K (K - 1)\bigr)$.
\end{proof}

\begin{example} Let $A$ be a $3$ blocks PABM with latent vectors
  $\{\lambda^{(k \ell)} \colon 1 \leq k \leq 3, 1 \leq \ell \leq 3\}$. Using the same notation as in Theorem
  \ref{theorem2}, we can define
\begin{gather*}
X = \begin{bmatrix}
\lambda^{(11)} & \lambda^{(12)} & \lambda^{(13)} & 0 & 0 & 0 & 0 & 0 & 0 \\
0 & 0 & 0 & \lambda^{(21)} & \lambda^{(22)} & \lambda^{(23)} & 0 & 0 & 0 \\
0 & 0 & 0 & 0 & 0 & 0 & \lambda^{(31)} & \lambda^{(32)} & \lambda^{(33)}
\end{bmatrix}, \\
Y = \begin{bmatrix}
\lambda^{(11)} & 0 & 0 & \lambda^{(12)} & 0 & 0 & \lambda^{(13)} & 0 & 0 \\
0 & \lambda^{(21)} & 0 & 0 & \lambda^{(22)} & 0 & 0 & \lambda^{(23)} & 0 \\
0 & 0 & \lambda^{(31)} & 0 & 0 & \lambda^{(32)} & 0 & 0 & \lambda^{(33)}
\end{bmatrix}.
\end{gather*}
Then $Y = X \Pi$ and $\tilde{P} = X Y^{\top}$ where $\Pi$ is a $9 \times 9$ 
permutation matrix of the form
$$\Pi = 
\Bigl[e_1 \mid e_4 \mid e_7 \mid e_2 \mid e_5 \mid e_8 \mid e_3
\mid e_6 \mid e_9 \Bigr].$$
where $e_i$ denotes the $i^{th}$ basis vector in $\mathbb{R}^{9}$. 
The matrix $\Pi$ corresponds to a permutation of $\{1,2,\dots,9\}$
with the following decomposition.
\begin{enumerate}
\item Positions 1, 5, 9 are fixed.
\item There are three cycles of length 2, namely $(2, 4)$, $(3, 7)$, and $(6, 8)$.
\end{enumerate}
We can thus write $\Pi$ as $\Pi = U I_{6, 3} U^\top$ where the first three
columns of $U$ consist of $e_1$, $e_5$, and $e_9$ corresponding to the
fixed points, the next three columns are the eigenvectors
$(e_k + e_{\ell}) / \sqrt{2}$, and the last three columns are the eigenvectors
$(e_k - e_{\ell}) / \sqrt{2}$ for $(k, \ell) \in
\{(2,4),(3,7),(6,8)\}$.

The matrix $\tilde{P}$ is then the edge probabilities matrix for a 
Generalized Random Dot Product Graph whose latent positions 
are the rows of the matrix
$$XU = \begin{bmatrix}
  \lambda^{(11)} & 0 & 0 &
  \frac{\lambda^{(12)}}{\sqrt{2}} & \frac{\lambda^{(13)}}{\sqrt{2}} & 0 &
  \frac{\lambda^{(12)}}{\sqrt{2}} & \frac{\lambda^{(13)}}{\sqrt{2}} & 0 \\
  0 & \lambda^{(22)} & 0 &
  \frac{\lambda^{(21)}}{\sqrt{2}} & 0 & \frac{\lambda^{(23)}}{\sqrt{2}} &
  -\frac{\lambda^{(21)}}{\sqrt{2}} & 0 & \frac{\lambda^{(23)}}{\sqrt{2}} \\
  0 & 0 & \lambda^{(33)} &
  0 & \frac{\lambda^{(31)}}{\sqrt{2}} & \frac{\lambda^{(32)}}{\sqrt{2}} &
  0 & -\frac{\lambda^{(31)}}{\sqrt{2}} & -\frac{\lambda^{(32)}}{\sqrt{2}}
\end{bmatrix}$$
and the latent positions for $P$ is a permutation of the rows of
$XU$. 
\end{example}

\hypertarget{methods}{%
\section{Algorithms}\label{methods}}

Two inference objectives arise from the PABM:

\begin{enumerate}
\def\labelenumi{\arabic{enumi}.}
\tightlist
\item
  Community membership identification (up to permutation).
\item
  Parameter estimation (estimating \(\lambda^{(k \ell)}\)'s).
\end{enumerate}

In our methods, the data that are observed for estimation is the adjacency matrix, $A \sim \PABM(\{\lambda^{(k \ell)}\}_K, \rho_n)$, along with an assumed number of communities, $K$. 
To motivate our methods, we first consider community detection and parameter estimation in the case where we know the edge probability matrix $P$ beforehand, noting that community memberships and popularity parameters are not immediately discernible from $P$ itself. 
After establishing methods for community detection and parameter estimation from $P$, we use the consistency property of the ASE \citep{doi:10.1080/01621459.2012.699795,rubindelanchy2017statistical} to demonstrate that the same methods work for $A$ almost surely as $n \to \infty$.

\hypertarget{related-work}{%
\subsection{Previous Work}\label{related-work}}

\citet{307cbeb9b1be48299388437423d94bf1} 
used Modularity Maximization (MM) and the Extreme Points (EP)
algorithm (\cite{le2016}) for community detection and parameter
estimation. They were able to show that as the sample size increases,
the {\em proportion} of misclassified community labels (up to permutation)
goes to 0.

\citet{noroozi2019estimation} used Sparse Subspace Clustering (SSC) 
(\cite{5206547}) for community detection in the PABM. 
The SSC algorithm can be described as follows: 
Given \(X \in \mathbb{R}^{n \times d}\) with vectors
\(x_i^\top \in \mathbb{R}^d\) as rows of \(X\), the optimization problem
\(c_i = \argmin_{c} \|c\|_1\) subject to \(x_i = X^\top c\) and
\(c^{(i)} = 0\), where $c^{(i)}$ is the $i^{th}$ entry of $c$, is solved for each \(i \in [n]\). 
The solutions are collected into matrix
\(C = \bigl[ c_1 \mid \cdots \mid c_n \bigr]^\top\) to
construct an affinity matrix \(B = |C| + |C^\top|\). If each \(x_i\) lies
exactly on one of \(K\) subspaces, \(B\) describes an undirected graph
consisting of {\em at least} \(K\) disjoint subgraphs, i.e., \(B_{ij} = 0\) if \(x_i, x_j\) lie on different subspaces. 
The intuition here is that vectors that lie on the same subspace can be described as linear combinations of each other, assuming the number of vectors in the subspace is greater than the dimensionality of the subspace. 
Then once sparsity is enforced, for each $c_i$, its $j^{th}$ element $c_i^{(j)}$ is zero if $x_j$ belongs to a subspace that doesn't contain $x_i$, resulting in $B_{ij} = 0$. 
If \(X\) instead represents points near \(K\) subspaces with some noise, 
then this property may only hold approximately and a final graph partitioning step may be required 
(e.g., edge thresholding or spectral clustering).

In practice, due to presence of noise, SSC is often done by solving the LASSO problems
\begin{equation} \label{eq:ssc}
c_i = \argmin_c \frac{1}{2} \|x_i - X_{-i}^\top c\|^2_2 + \vartheta \|c\|_1
\end{equation}
for some sparsity parameter \(\vartheta > 0\). 
The \(c_i\) vectors are then collected into \(C\) and \(B\) as before.

\begin{definition}[Subspace Detection Property]
  \label{def:subspace_detection}
Let $X = \bigl[ x_1 \mid \cdots \mid x_n \bigr]^\top$ be noisy
points sampled from $K$ subspaces, i.e., $x_i = y_i + z_i$ where the
$y_i$ belongs to the union of $K$ subspaces and the $z_i$ are noise
vectors. Let $\vartheta \geq 0$ be given and
let $C$ and $B$ be constructed from the
solutions of LASSO problems as described in Eq.~\eqref{eq:ssc} with
this given choice of $\vartheta$. Then $X$
is said to satisfy the subspace detection property 
with sparsity parameter $\vartheta$ if each column of
$C$ has nonzero $\ell_2$ norm and $B_{ij} = 0$ whenever $y_i$ and $y_j$ are
from different subspaces.
\end{definition}

\begin{remark}
One common approach to show that SSC works for a noisy sample $X$ is to show that $X$ satisfies the subspace detection property for some choice of $\vartheta$; 
recall that $\vartheta$ is the sparsity parameter for
the LASSO problems in Eq.~\eqref{eq:ssc}. However, this is not sufficient to
guarantee that SSC perfectly recovers the underlying subspaces.
More specifically, if $X$ satisfies the 
subspace detection property, 
then $B$ describes a graph with {\em at least} $K$ disconnected subgraphs, 
with the ideal case being that there are exactly $K$ subgraphs 
which map onto each subspace. 
Nevertheless it is also possible that the $K$ subspaces are represented by
$K' > K$
multiple disconnected subgraphs and we cannot, at least without a subsequent
post-processing step, recover the $K$ subspaces directly from $B$;
see \citet{sdp_sufficiency} and \citet{liu_ssc}
for further discussions. Therefore in practice $B$ is usually treated as an
affinity matrix and, as we allude to earlier, 
the rows of $B$ are partitioned using some
clustering algorithm to obtain the final clustering. 
\end{remark}

Theorem~\ref{theorem2} suggests that SSC is appropriate for community
detection for the PABM, provided that we observe the edge probabilities
matrix $P$. More precisely, given the matrix $\tilde{P}$ obtained by
permuting the rows and columns of $P$ as described in
Remark~\ref{rem:pabm_view2} we can recover $XU$ up
to some non-identifiability indefinite orthogonal transformation $Q$. 
Then using results from \citet{soltanolkotabi2012}, it can be easily shown
that the subspace detection property holds for \(XU\). Indeed, the columns
of \(XU\) from different communities correspond to mutually
orthogonal subspaces. This then implies that the
subspace detection property also holds for $XUQ$ for all invertible
transformation $Q$ and hence the subspace detection property also holds for
$\tilde{\Pi} X U Q$ for any $n \times n$ permutation matrix $\tilde{\Pi}$. 

However, because we do not observe $P$ but rather only the noisy
adjacency matrix $A \sim \BG(P)$, the natural approach then is
to perform SSC on the rows of the spectral embedding of $A$, 
since the embedding of $P$ consists of $K$ subspaces (Theorem~\ref{theorem2}), 
and so the embedding of $A$ will lie approximately on the $K$ subspaces. 
We will show in Theorem~\ref{theorem5} that, with probability converging to
one as $n \rightarrow \infty$, the rows of the ASE of $A$ also
satisfy the subspace detection property. 
Theorem~\ref{theorem5} builds upon existing work by
\citet{rubindelanchy2017statistical} who describe the convergence
behavior of the ASE of \(A\) to that of \(\tilde{\Pi} XU\), and
\citet{jmlr-v28-wang13} who show the necessary conditions for the
subspace detection property to hold in noisy cases where the points lie
near subspaces. Finally we emphasize that while
\citet{noroozi2019estimation} also considered the use of SSC for
community recovery in PABM, they instead applied SSC to the rows of
\(A\) itself, foregoing
the embedding step altogether. It is however much harder to show that
the rows of $A$ satisfy the subspace detection property and thus, to
the best of our knowledge, there is currently no consistency result
regarding the application of SSC to the rows of $A$. 

% Using results from \citet{soltanolkotabi2012}, it can be easily shown
% that the subspace detection property holds for \(XU\), which is an ASE
% of \(P\). More specifically, if points lie exactly on mutually
% orthogonal subspaces, then the subspace detection property will hold
% with probability 1, and this is exactly the case for the PABM (Theorem
% \ref{theorem2}). Much of our work is then built on
% \citet{rubindelanchy2017statistical}, who describe the convergence
% behavior of the ASE of \(A\) to the ASE of \(P\), and
% \citet{jmlr-v28-wang13}, who show the necessary conditions for the
% subspace detection property to hold in noisy cases where the points lie
% near subspaces.
\hypertarget{community-detection}{%
\subsection{Algorithms for Community Detection}\label{community-detection}}

We previously stated in Theorem~\ref{theorem2} one possible set of latent positions that result in
the edge probability matrix of a PABM, namely
\(P = \tilde{\Pi} (XU) I_{p, q}
(XU)^\top \tilde{\Pi}^{\top}\) where $X$ is block diagonal and
$\tilde{\Pi}$ is a permutation matrix.  
Furthermore, the explicit form of \(XU\) represents points in \(\mathbb{R}^{K^2}\)
such that points within each community lie on \(K\)-dimensional
orthogonal subspaces, i.e. $\langle U^{\top} x_i, U^{\top} x_j \rangle = 0$ whenever $i$ and $j$ are in different communities. 
Thus if we have (or can estimate) \(XU\) directly, then both the community
detection and parameter identification problem are trivial because \(U\)
is orthonormal and fixed for each value of \(K\). 
However, direct
identification or estimation of \(XU\) is possibly difficult 
due to the non-identifiability of $XU$ (see Remark~\ref{rem:non_identifiable})
when we are given only $P$. 
More specifically, suppose we find a matrix $Y \in
\mathbb{R}^{n \times K^2}$
such that \(P = Y I_{p, q} Y^\top\). Then it is generally the case that
\( Y = \tilde{\Pi} XU Q \) for some indefinite orthogonal matrix
$Q \in \mathbb{O}(p,q)$. 
However since \(Q\) is not necessarily an
orthogonal matrix and hence, if $y_i$ denote the $i^{th}$ row of $Y$, 
then $\langle U^{\top} x_i, U^{\top} x_j \rangle
\neq \langle y_i, y_j \rangle$.
This prevents us from transferring the orthogonality property of
\(XU\) directly to $Y$. 

Nevertheless by using the special geometric structure of $X$ we can circumvent the
non-identifiability of $Y$ and $XU$ by using instead the rows of the
matrix $V$ of eigenvectors (corresponding to the non-zero eigenvalues) of $P$. In particular $V$ is identifiable
up to orthogonal transformations and furthermore, due to the block
diagonal structure of $X$, the rows of $V$ also lie on $K$ distinct orthogonal
subspaces and hence $v_i^{\top} v_j = 0$ whenever $z_i \not = z_j$. 

\begin{theorem}
\label{theorem3}
Let $P = V D V^\top$ be the spectral decomposition 
of the edge probability matrix. 
Let $B = n V V^\top$. 
Assume $\lambda_{i z_i} > 0$ for each $i \in [n]$, 
i.e., each vertex's popularity parameter to its own community is nonzero. 
Then $B_{ij} = 0$ if and only if 
vertices $i$ and $j$ are in different communities.
\end{theorem}

\begin{proof}
We first show that $V V^\top =
\tilde{\Pi} X (X^\top X)^{-1} X^\top \tilde{\Pi}^{\top} $ where $X$ is 
defined as in Eq.~\eqref{eq:xy}. Indeed, by Theorem 2, 
\(P = \tilde{\Pi} X U I_{p, q} U^\top X^\top \tilde{\Pi}\) for $p = K(K+1)/2$ and $q = K(K-1)/2$. 
The eigendecomposition \(P = V D V^\top\) also yields $P = V
|D|^{1/2} I_{p, q} |D|^{1/2} V^\top$ where \(|\cdot|^{1/2}\) is
applied entry-wise. Now let $Y = \tilde{\Pi} XU$ and $\tilde{Y} = V|D|^{1/2}$; note that
$Y$ and $\tilde{Y}$ both have full column ranks. 
Because $P = Y I_{p,q} Y^{\top} = \tilde{Y} I_{p,q} \tilde{Y}^{\top}$, we have
$$Y = \tilde{Y} I_{p,q} \tilde{Y}^{\top} Y (Y^{\top} Y)^{-1} I_{p,q}.$$
Let $Q = I_{p,q} \tilde{Y}^{\top} Y (Y^{\top} Y)^{-1} I_{p,q}$ and note that
$Y = \tilde{Y} Q$. We then have
\begin{equation*}
  \begin{split}
  Q^{\top} I_{p,q} Q &= I_{p,q} (Y^{\top} Y)^{-1} Y^{\top} \tilde{Y} I_{p,q}
I_{p,q} I_{p,q} \tilde{Y}^{\top} Y (Y^{\top} Y)^{-1} I_{p,q} \\ 
&= I_{p,q} (Y^{\top} Y)^{-1} Y^{\top} Y I_{p,q}
Y^{\top} Y (Y^{\top} Y)^{-1} I_{p,q} =  I_{p,q}
\end{split}
\end{equation*}
and hence $Q$ is an indefinite orthogonal matrix. 

Let $R = U Q |D|^{-1/2}$ and note that $V = \tilde{\Pi} XR$. Because $R$ is
invertible, we can write
$$\tilde{\Pi} X (X^{\top} X)^{-1} X^{\top} \tilde{\Pi}^{\top} =
\tilde{\Pi} X R (R^{\top} X^{\top} X R)^{-1}
R^{\top} X^{\top} \tilde{\Pi}^{\top}.$$ 
Furthermore, as $V$ has orthonormal columns, $R^{\top} X^{\top} X R =
V^{\top} \tilde{\Pi} \tilde{\Pi}^{\top} V = V^{\top} V = I$. We thus conclude
$$\tilde{\Pi} X (X^{\top} X)^{-1} X^{\top} \tilde{\Pi}^{\top} = V (V^{\top} V)^{-1} V^{\top} = V V^{\top}$$
as desired.

To complete the proof of Theorem~\ref{theorem3}, recall that \(X\) 
is block diagonal with each block corresponding to one community, 
and hence \(X (X^\top X)^{-1} X^\top\) is also a
block diagonal matrix with each block corresponding to a community. 
As $B = n VV^{\top} = n \tilde{\Pi} X (X^\top X)^{-1} X^\top \tilde{\Pi}^{\top}$, 
we conclude that $B_{ij} = 0$ 
whenever vertices $i$ and $j$ belong to different communities.  
\end{proof}

Theorem \ref{theorem3} provides perfect community detection from \(P\).
More specifically, let \(|B|\) be the affinity matrix for graph \(G'\), 
where $|\cdot|$ is applied entry-wise. Then
\(G'\) consists of exactly \(K\) disjoint subgraphs, 
as $G'$ has no edges between communities. 
All that is left to identify the communities is 
to assign each subgraph a distinct community label. 
In practice, we do not observe $P$ and instead only observe the noisy
$A \sim \BG(P)$. A natural approach is then to use
the affinity matrix $\hat{B} = n \hat{V} \hat{V}^{\top}$ where
$\hat{V}$ is the matrix of eigenvectors (corresponding to the largest
eigenvalues in modulus) of $A$. The resulting procedure, named
Orthogonal Spectral Clustering, is presented in
Algorithm~\ref{alg:osc}.
The following result leverages existing theoretical properties
of ASE for estimating of latent positions in a GRDPG \citep{rubindelanchy2017statistical} to show that
$\hat{B}$ converges almost surely to $B$; in particular 
 $\hat{B}_{ij} \stackrel{a.s.}{\to} 0$ 
for each pair $(i, j)$ in different communities. 

\begin{algorithm}[tp]
  \label{alg:osc}
  \DontPrintSemicolon
  \SetAlgoLined
  \KwData{Adjacency matrix $A$, number of communities $K$}
  \KwResult{Community assignments $1, ..., K$}
    Compute the eigenvectors of $A$ that correspond to the $K (K+1) / 2$ most
    positive eigenvalues and $K (K-1) / 2$ most negative eigenvalues. Construct
    $V$ using these eigenvectors as its columns.\;
    Compute $B = |n V V^\top|$, applying $|\cdot|$ entry-wise.\;
    Construct graph $G$ using $B$ as its similarity matrix.\;
    Partition $G$ into $K$ disconnected subgraphs
    (e.g., using edge thresholding or spectral clustering).\;
    Map each partition to the community labels $1, ..., K$.\;
  \caption{Orthogonal Spectral Clustering.}
\end{algorithm}

\begin{theorem}
\label{theorem4}
Assume the setting of Theorem~\ref{theorem3}. 
Let $\hat{B}$ with entries $\hat{B}_{ij}$ be the affinity matrix
obtained from OSC as described in Algorithm~\ref{alg:osc}. 
Then for $n \rho_n = \omega(\log^{4}{n})$, we have

\begin{equation} \label{eq:thm4a}
\max_{i, j} |\hat{B}_{ij} - B_{ij}| = O\Big( \frac{\log n}{\sqrt{n \rho_n}} \Big)
\end{equation}

with high probability. In particular 
$\hat{B}_{ij} -B_{ij} \overset{\mathrm{a.s.}}{\rightarrow} 0$ 
where the convergence is
uniform over all $i,j$. Hence for all pairs $(i,j)$ in different
communities we have 
$\hat{B}_{ij} \overset{\mathrm{a.s.}}{\rightarrow} 0$, 
while for all pairs $(i, j)$ in the same community, 
$\liminf_{n \to \infty} |\hat{B}_{ij}| > 0$ almost surely. 
\end{theorem}

Theorem~\ref{theorem4} guarantees that for any $\epsilon > 0$, 
the number of edges of $\hat{B}$ between vertices of different communities 
that are larger than $\epsilon$ converges to zero with probability
converging to one as $n$ increases. 
We can always find an $\epsilon > 0$ such that $\hat{B}_{ij} > \epsilon$ 
with probability converging to one as $n$ increases. 
Thus, by using $\hat{B}$, we can 
perfectly recover all the latent community assignments $z_1, z_2,
\dots, z_n$, i.e., the number of misclustered vertices is zero
asymptotically almost surely. We note that Theorem~\ref{theorem4} is
stronger than existing results in the literature; in particular
Theorem~1 of \citet{307cbeb9b1be48299388437423d94bf1} (the paper that
originally introduces the PABM model) only guarantees that the {\em proportion} of misclustered vertices converges to $0$ as $n
\rightarrow \infty$. Furthermore Theorem~1
of \citet{307cbeb9b1be48299388437423d94bf1} also requires the sparsity
parameter $\rho_n$ to satisfies $n \rho_n^2 = \omega(\log^2{n})$ which
is a considerably stronger assumption than the assumption $n \rho_n =
\omega(\log^{4}{n})$ used in Theorem~\ref{theorem4}. Indeed, $n
\rho_n^2 = \omega(\log^{2}{n})$ implies $n \rho_n = \omega(n^{1/2})$. 
We emphasize that the assumption $n \rho_n = \omega(\log^{c}{n})$ for
some constant $c > 1$ is commonly used in the context of graph
estimation using spectral methods.

Theorems \ref{theorem2}, \ref{theorem3}, and \ref{theorem4} also provide
a natural path toward using SSC for community detection. 
In particular we established in Theorem \ref{theorem2} that an ASE of the edge
probability matrix \(P\) can be constructed from a latent vector configuration 
consisting of orthogonal subspaces. Theorem \ref{theorem3} shows how 
this property can also be recovered from the eigenvectors of \(P\). 
Then Theorem \ref{theorem4} shows that, by replacing $P$ with $A$, the
rows of $\hat{V}$ also lie on asymptotically orthogonal subspaces.
Motivated by Theorem~\ref{theorem4}, Theorem \ref{theorem5} below
shows that the subspace detection property also holds for the rows of
$\sqrt{n} \hat{V}$. 
\begin{algorithm}[t]
  \label{alg:ssc}
  \DontPrintSemicolon
  \SetAlgoLined
  \caption{Sparse Subspace Clustering using LASSO.}
  \KwData{Adjacency matrix $A$, number of communities $K$,
  hyperparameter $\lambda$}
  \KwResult{Community assignments $1, ..., K$}
    Find $V$, the matrix of eigenvectors of $A$
    corresponding to the $K (K + 1) / 2$ most positive
    and the $K (K - 1) / 2$ most negative eigenvalues.\;
    Normalize $V \leftarrow \sqrt{n} V$.\;
    \For {$i = 1, ..., n$} {
      Assign $v_i^\top$ as the $i^{th}$ row of $V$.
      Assign $V_{-i} = \bigl[
      v_1 \mid \cdots \mid v_{i-1} \mid v_{i+1} \mid \cdots \mid v_n \bigr]^\top$.\;
      Solve the LASSO problem
      $c_i = \arg\min_{\beta}
      \frac{1}{2} \|v_i - V_{-i} \beta\|_2^2 + \lambda \|\beta\|_1$.\;
      Assign $\tilde{c}_i = (c_i^{(1)}, \dots, c_i^{(i-1)}, 0, c_i^{(i)}, \dots, c_i^{(n-1)})^\top$ such that the superscript is the index of
      $\tilde{c}_i$.\;
    }
    Assign
    $C = \bigr[ \tilde{c}_1 \mid \cdots \mid \tilde{c}_n \bigr]$.\;
    Compute the affinity matrix $B = |C| + |C^\top|$.\;
    Construct graph $G$ using $B$ as its similarity matrix.\;
    Partition $G$ into $K$ disconnected subgraphs (e.g., using edge
    thresholding or spectral clustering).\;
    Map each partition to the community labels $1, ..., K$.
\end{algorithm}

\begin{theorem}
\label{theorem5}
Let $P$ describe the edge probability matrix of the PABM with
$n$ vertices, and let $A \sim \Bernoulli(P)$.  Let $\hat{V}$ be the
matrix of eigenvectors of $A$ corresponding to the $K^2$ largest
eigenvalues in modulus. Then for any $\epsilon > 0$ 
there exists a choice of $\vartheta > 0$ and $N \in \mathbb{N}$ such
that for all $n \geq N$, $\sqrt{n} \hat{V}$ obeys the subspace detection property with
probability at least $1 - \epsilon$.  
\end{theorem}

\hypertarget{parameter-estimation}{%
\subsection{Algorithm for Parameter Estimation}\label{parameter-estimation}}
For ease of exposition we now assume in this subsection that the
edge probability matrix \(P\) for the PABM had been arranged so that the rows
and columns are organized by community so that $\tilde{P} = P$ (see
Remark~\ref{rem:pabm_view2}). Then the \(k\ell\)\textsuperscript{th}
block is an outer product of two vectors, i.e.,
\(P^{(k \ell)} = \lambda^{(k \ell)} (\lambda^{(\ell k)})^\top\). Therefore, given
\(P^{(k \ell)}\), \(\lambda^{(k \ell)}\) and \(\lambda^{(\ell k)}\) are solvable
up to multiplicative constant using singular value
decomposition. More specifically 
let \(P^{(k \ell)} = (\sigma^{(k \ell)})^2 u^{(k \ell)} (v^{(k \ell)})^\top\)
be the singular value decomposition of \(P^{(k \ell)}\) where
\(u^{(k \ell)} \in \mathbb{R}^{n_k}\) and 
\(v^{(k \ell)} \in \mathbb{R}^{n_\ell}\) are vectors
and \(\sigma^{(k \ell)}\) is a scalar. 
Then \(\rho_n^{1/2} \lambda^{(k \ell)} = s_1 u^{(k \ell)}\)
and \(\rho_n^{1/2} \lambda^{(\ell k)} = s_2 v^{(k \ell)}\) 
for unidentifiable $s_1 s_2 = (\sigma^{(k \ell)})^2$.
Because each $\lambda^{(k \ell)}$ is not strictly identifiable,
we instead estimate each 
$\tilde{\lambda}^{(k \ell)} = \sigma^{(k \ell)} u^{(k \ell)}$. 
Given the adjacency matrix \(A\)
instead of edge probability matrix \(P\), we can simply use plug-in
estimators by taking the SVD of each $A^{(k \ell)}$ to obtain 
$\hat{\lambda}^{(k \ell)} = \hat{\sigma}^{(k \ell)} \hat{u}^{(k \ell)}$ 
using the largest singular value of $A$ and its corresponding singular vectors. 

\begin{algorithm}[tp]
  \label{alg:param_est}
  \DontPrintSemicolon
  \SetAlgoLined
  \caption{PABM parameter estimation.}
  \KwData{Adjacency matrix $A$, community assignments $1, ..., K$}
  \KwResult{PABM parameter estimates $\{\hat{\lambda}^{(k \ell)}\}_K$.}
  Arrange the rows and columns of $A$ by community such that each 
  $A^{(k \ell)}$ block consists of estimated edge probabilities between 
  communities $k$ and $l$.\;
  \For {$k, \ell = 1, ..., K$, $k \leq \ell$} {
    Compute $A^{(k \ell)} = U \Sigma V^\top$, the SVD of the $k\ell$-th 
    block.\;
    Assign $u^{(k \ell)}$ and $v^{(k \ell)}$ as the first columns of $U$ and $V$. 
    Assign $(\sigma^{(k \ell)})^2 \leftarrow \Sigma_{11}$.\;
    Assign $\hat{\lambda}^{(k \ell)} \leftarrow \pm \sigma^{(k \ell)} u^{(k \ell)}$ and 
    $\hat{\lambda}^{(\ell k)} \leftarrow \pm \sigma^{(k \ell)} v^{(k \ell)}$.
  }
\end{algorithm}

\begin{theorem}
\label{theorem6}
Let each $\tilde{\lambda}^{(k \ell)}$ be the popularity vector derived
from its corresponding $P^{(k \ell)}$ and let $\hat{\lambda}^{(k
  \ell)}$ be its estimate obtained from $A^{(k \ell)}$ using Algorithm~\ref{alg:param_est}.
Then if $n \rho_n = \omega( \log^{4}{n})$,
\begin{equation} \label{eq:thm6}
\max_{k, \ell \in \{1, ..., K\}} 
\|\hat{\lambda}^{(k \ell)} - \tilde{\lambda}^{(k \ell)}\|_{\infty} = 
O\bigg(\frac{\log n_k}{\sqrt{n_k}} \bigg)
\end{equation}
with high probability. Here $\|\cdot\|_{\infty}$ denotes the
$\ell_\infty$ norm of a vector. Let $\hat{\Lambda}$ be the matrix

$$\hat{\Lambda} = \begin{bmatrix} \hat{\lambda}^{(11)} &
  \hat{\lambda}^{(12)} & \cdots & \hat{\lambda}^{(1 K)} \\
  \hat{\lambda}^{(21)} &
  \hat{\lambda}^{(22)} & \cdots & \hat{\lambda}^{(2 K)} \\
  \vdots & \vdots & \cdots & \vdots \\
   \hat{\lambda}^{(K1)} &
  \hat{\lambda}^{(K2)} & \cdots & \hat{\lambda}^{(K K)} 
\end{bmatrix}$$

and let $\hat{P} = \hat{X} U I_{p,q} U^\top \hat{X}^{\top}$ 
where $\hat{X}$ is defined from $\hat{\Lambda}$ and $U$ is defined from $K$ as in Theorem~\ref{theorem2}. 
Eq.~\eqref{eq:thm6} then implies
\begin{equation}
\label{eq:conv_prob}
\frac{1}{n} \|\rho_n^{-1} \hat{P} - \rho_n^{-1} P\|_{F} = O((n
\rho_n)^{-1/2}), \quad \max_{ij} |\rho_n^{-1} \hat{P}_{ij} -
\rho_n^{-1} P_{ij}| = O((n \rho_n)^{-1/2})
\end{equation}
with high probability. \end{theorem}
Eq.~\eqref{eq:thm6} guarantees that $n^{-1/2} \|\rho_n^{-1/2} \hat{\Lambda} -
\Lambda\|_{F} = O((n \rho_n)^{-1/2})$. Eq.~\eqref{eq:conv_prob} then guarantees that
the mean square error for $\rho_n^{-1} (\hat{P} - P)$ converges to $0$ almost surely and
furthermore the entries of $\rho_n^{-1}
\hat{P}$ converge uniformly to the entries of $\rho_n^{-1} P$; recall that
$\rho_n^{-1} P_{ij} = \lambda_{iz_j} \lambda_{j z_i}$. We note that
these results are stronger than existing results in
\citet{307cbeb9b1be48299388437423d94bf1}; for example Theorem~2 in
\citet{307cbeb9b1be48299388437423d94bf1} only guarantees $n^{-1/2} \|\rho_n^{-1/2} \hat{\Lambda} -
\Lambda\|_{F} = o(1)$ as $n \rightarrow \infty$.

\hypertarget{simulated-examples}{%
\section{Simulation Study}\label{simulated-examples}}

For each simulation, community labels are drawn from a multinomial
distribution, the popularity vectors \(\{\lambda^{(k \ell)}\}_K\) are drawn
from two types of joint distributions depending on whether \(k =
\ell\) or $k \not = \ell$. The edge probability matrix \(P\) is constructed using the popularity
vectors and finally the adjacency matrix \(A\)
is drawn \(A \sim \mathrm{Bernoulli}(P)\). OSC (Algorithm~\ref{alg:osc}) is then used for community detection, and this
method is compared against (1) SSC using the spectral embedding $\hat{V}$
(Algorithm~\ref{alg:ssc}), (2) SSC using the rows of the
observed adjacency matrix $A$ as is done in \citet{noroozi2019estimation}
and (3) modularity maximization (MM) as is done in
\citet{307cbeb9b1be48299388437423d94bf1}. We denote the two SSC
implementations using the rows of $A$ and using the spectral embedding
of $A$ as SSC-A and SSC-ASE, respectively. 
The parameters \(\vartheta\)
that controls the sparsity for SSC-A and SSC-ASE were chosen via a preliminary cross-validation experiment. 
In practice, $\vartheta$ that guarantee SDP (if it is possible for the particular simulated data) often result in more than $K$ disconnected subgraphs, so a smaller $\vartheta$ that does not guarantee SDP was chosen, 
and the final clustering step of SSC-A and SSC-ASE was done
by fitting a Gaussian Mixture Model to the normalized Laplacian
eigenmap embeddings \citep{belkin03:_laplac} of the affinity matrix \(B\).
We also estimate the latent popularity vectors $\{\lambda^{(k \ell)}\}$
by assuming that the true community labels are known and then apply
Algorithm~\ref{alg:param_est}, and we compare this estimation method against an
MLE-based estimator as described in \citet{noroozi2019estimation} and
\citet{307cbeb9b1be48299388437423d94bf1}.

Modularity Maximization is NP-hard, so
\citet{307cbeb9b1be48299388437423d94bf1} used the Extreme Points
(EP) algorithm (\cite{le2016}) as a greedy
relaxation of the optimization problem; the EP algorithm has a running
time of $O(n^{K-1})$ where $n$ is the number of vertices in the graph
and $K$ is the number of communities.
For these simulations we instead replace the EP algorithm with the
Louvain algorithm for modularity maximization,
as the implementation of the EP algorithm in
\citet{307cbeb9b1be48299388437423d94bf1} is too computationally expensive for \(K > 2\). For \(K = 2\), it
was verified that the Louvain algorithm produces comparable results
to EP-MM.

For comparing methods, we define the community detection error as:
\[L_c(\hat{\sigma}, \sigma; V) =
\min_\pi \sum_i I(\pi \circ \hat{\sigma}(v_i) = \sigma(v_i))\]
where \(\sigma(v_i)\) is the true community label of vertex \(v_i\),
\(\hat{\sigma}(v_i)\) is the predicted label of \(v_i\), and \(\pi\) is
a permutation operator. This is effectively the ``misclustering count''
of clustering function \(\hat{\sigma}\).

For parameter estimation, because the popularity parameters
$\{\lambda_{ik}\}$ are unidentifiable, we instead estimate the edge
probabilities $P_{ij} = \lambda_{i z_j} \lambda_{j z_i}$ via the
quantities $\hat{P}_{ij} = \hat{\lambda}_{iz_j} \hat{\lambda}_{jz_i}$. The
parameter estimation error is then given by the normalized 
Frobenius norm of $P$ divided by the number of vertices, i.e.,
$$\mathrm{RMSE}(\hat{P}, P) = \frac{1}{n} \|\hat{P} - P\|_F.$$

We also note that unlike the MLE-based method
\citep{307cbeb9b1be48299388437423d94bf1}, the ASE method in
Algorithm~\ref{alg:param_est} can be trivially modified so as to not
require the community labels if we are only interested in 
estimating $P$. More specifically we first compute the ASE
$\hat{Z}$ of $A$ (see Remark~\ref{rem:ase}) and then compute $\hat{P}
= \hat{Z} I_{p,q} \hat{Z}^{\top}$. The resulting estimate $\hat{P}$
will have the same convergence rate as that given in Eq.~(\ref{eq:conv_prob}).

\hypertarget{balanced-communities}{%
\subsection{Balanced Communities}\label{balanced-communities}}

In each simulation, community labels \(z_1, ..., z_n\) were drawn from a
multinomial distribution with mixture parameters
\(\{\alpha_1, ..., \alpha_K\}\), then \(\{\lambda^{(k \ell)}\}_K\) according
to the drawn community labels, \(P\) was constructed using the drawn
\(\{\lambda^{(k \ell)}\}_K\), and \(A\) was drawn from \(P\).

For these examples, we set the following parameters:

\begin{itemize}
\tightlist
\item
  Number of vertices \(n = 128, 256, 512, 1024, 2048, 4096\)
\item
  Number of underlying communities \(K = 2, 3, 4\)
\item
  Mixture parameters \(\alpha_k = 1 / K\) for \(k = 1, ..., K\), (i.e.,
  each community label has an equal probability of being drawn)
\item
  Community labels
  \(z_k \stackrel{\text{iid}}{\sim} \Multinomial(\alpha_1, ..., \alpha_K)\)
\item
  Within-group popularities
  \(\lambda^{(kk)} \stackrel{\text{iid}}{\sim} \Betadist(2, 1)\)
\item
  Between-group popularities
  \(\lambda^{(k \ell)} \stackrel{\text{iid}}{\sim} \Betadist(1, 2)\) for
  \(k \neq \ell\)
\end{itemize}
Fifty simulations were performed for each combination of $n$ and
$K$. The results for community recovery and parameter estimations are
presented in Fig.~\ref{fig:clust_err_ct_sim} and
Fig.~\ref{fig:p_block_est}, respectively.

Fig.~\ref{fig:clust_err_ct_sim} shows that OSC recovers the community
perfectly as $n$ increases, i.e., the number of mislabeled vertices
goes to $0$. The performance of SSC-ASE is comparable to OSC for $K
\geq 3$ but is noticeably worse when $K = 2$. Similarly, SSC on both the embedding and on the
adjacency matrix produces similar trends for \(K > 2\). The difference
in performance between SSC-A and SSC-ASE for \(K = 2\) can be attributed to the final spectral
clustering step of the affinity matrix. While the subspace detection
property is guaranteed for large $n$, in our simulations, setting the
sparsity parameter $\vartheta$ to the required value usually resulted
in more than $K$ disconnected subgraphs in the affinity matrix
$\hat{B}$.  
% More specifically the sparsity parameter to the required value resulted
% in more than \(K\) disconnected subgraphs. 
We instead chose a smaller sparsity parameter, 
necessitating a final clustering step. 
A GMM was fit to the normalized Laplacian eigenmap of $\hat{B}$,
but visual inspection suggests that the communities are not
distributed as a mixture of Gaussians in the eigenmap. 
A different choice of mixture distribution may result in better performance. 

\begin{figure}[tp]
{\centering \includegraphics{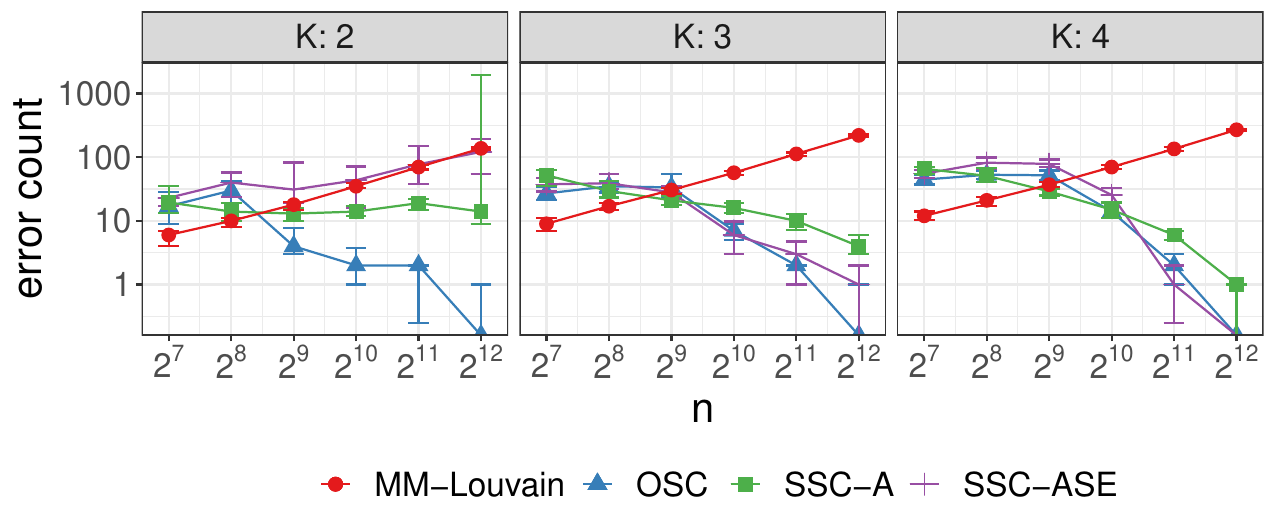}
}

\caption{Median and IQR of community detection error. Communities are approximately balanced. Simulations were repeated 50 times for each sample size.}\label{fig:clust_err_ct_sim}
\end{figure}

Given ground truth community labels, Fig.~\ref{fig:p_block_est} shows
that Algorithm~\ref{alg:param_est} and the MLE-based
plug-in estimators perform comparably, with root mean square
error decaying at rate approximately \(n^{-1/2}\) as $n$ increases.

\begin{figure}[tp]
{\centering \includegraphics{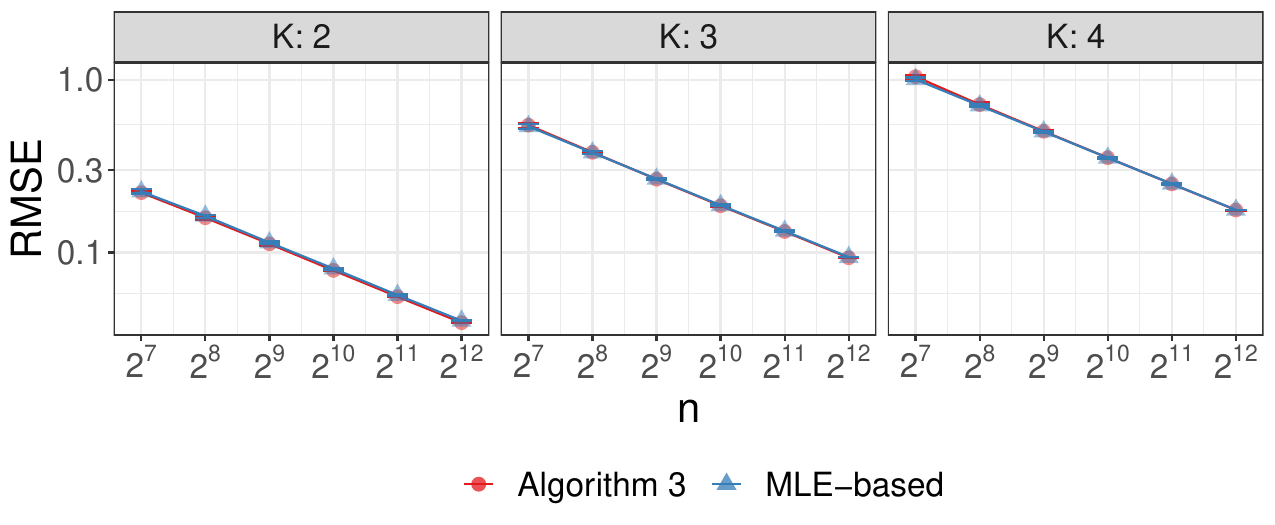}}
\caption{Median and IQR RMSE for edge probability matrices
  reconstructed from the outputs of Algorithm~\ref{alg:param_est}
  (red) compared against outputs of an MLE-based method (blue) proposed in \citet{307cbeb9b1be48299388437423d94bf1}.
  Simulations were repeated 50 times for each sample size. Communities were drawn to be approximately balanced.}
\label{fig:p_block_est}
\end{figure}

\hypertarget{imbalanced-communities}{%
\subsection{Imbalanced Communities}\label{imbalanced-communities}}

Simulations performed in this section are the same as those in the
previous section with the exception of the mixture parameters
\(\{\alpha_1, ..., \alpha_K\}\) used to draw community labels from the
multinomial distribution. For these examples, we set the following
parameters:

\begin{itemize}
\tightlist
\item
  Number of vertices \(n = 128, 256, 512, 1024, 2048, 4096\)
\item
  Number of underlying communities \(K = 2, 3, 4\)
\item
  Mixture parameters \(\alpha_k = \frac{k^{-1}}{\sum_{\ell=1}^K \ell^{-1}}\)
  for \(k = 1, ..., K\)
\item
  Community labels
  \(z_k \stackrel{\text{iid}}{\sim} \Multinomial(\alpha_1, ..., \alpha_K)\)
\item
  Within-group popularities
  \(\lambda^{(kk)} \stackrel{\text{iid}}{\sim} \Betadist(2, 1)\)
\item
  Between-group popularities
  \(\lambda^{(k \ell)} \stackrel{\text{iid}}{\sim} \Betadist(1, 2)\) for
  \(k \neq \ell\)
\end{itemize}
Fifty simulations were performed for each combination of $n$ and
$K$. The results for community recovery and parameter estimations are
presented in Fig.~\ref{fig:clust_err_ct_sim_imba} and Fig.~\ref{fig:lambda_est_p_imba}, respectively.

From Fig.~\ref{fig:clust_err_ct_sim_imba} we once again see that the
number of mislabeled vertices trending to $0$ for OSC. The
performance of SSC-ASE is comparable to that of OSC for $K > 2$ but
is worse when $K = 2$. Fig.~\ref{fig:lambda_est_p_imba} indicates that
the parameter estimation error also decays at rate $n^{-1/2}$ similar
to that in the balanced communities setting.

\begin{figure}[H]
{\centering \includegraphics{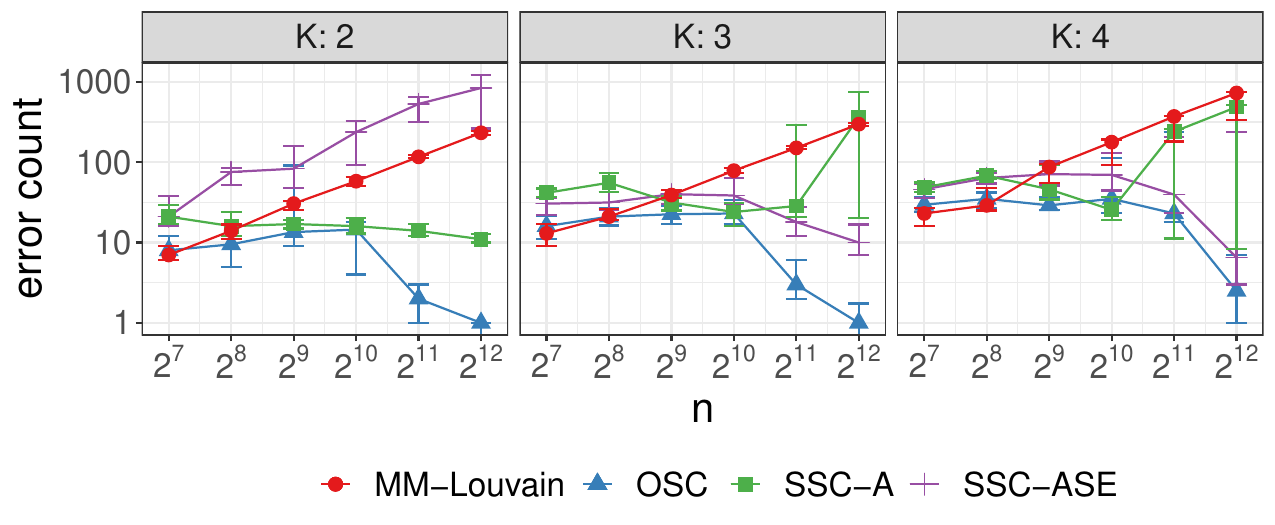}
}
\caption{Median and IQR of community detection error. Communities are imbalanced. 
Simulations were repeated 50 times for each sample size.}\label{fig:clust_err_ct_sim_imba}
\end{figure}

\begin{figure}[H]
{\centering \includegraphics{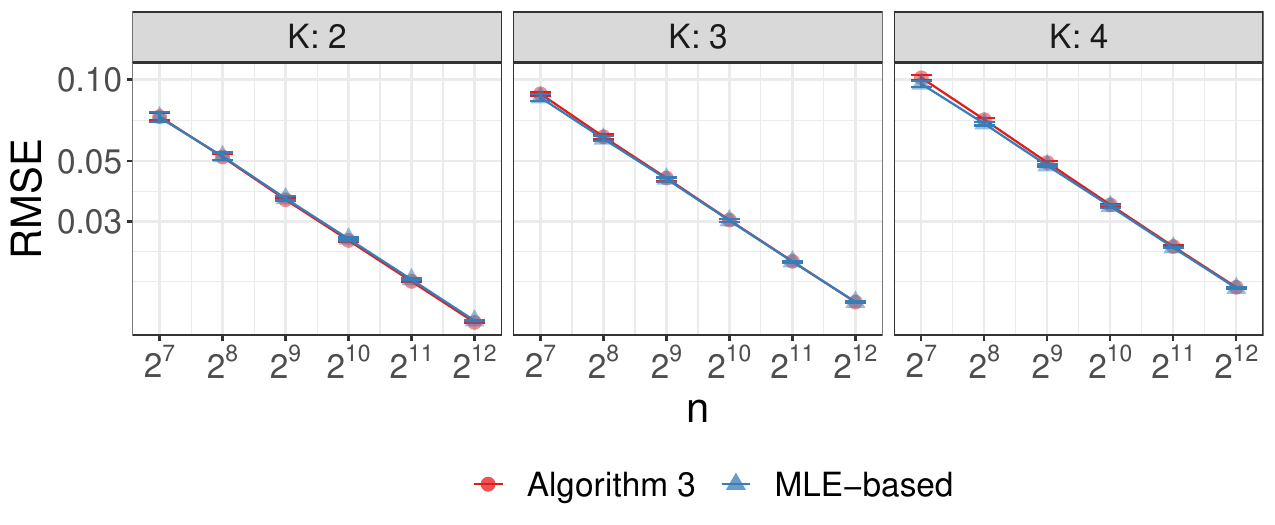}
}
\caption{Median and IQR RMSE of edge probabilities derived from the
  outputs of Algorithm \ref{alg:param_est} (red) compared against an
  MLE-based method (blue) described in
  \cite{307cbeb9b1be48299388437423d94bf1}.
  Simulations were repeated 50 times for each sample size. Communities were drawn to be imbalanced.}
\label{fig:lambda_est_p_imba}
\end{figure}

\hypertarget{disassortative-models}{%
\subsection{Disassortative Models}\label{disassortative-models}}

\citet{rubindelanchy2017statistical} demonstrated the power of applying GRDPG-based approaches to disassortative block models. 
Likewise, demonstrate that GRDPG-based algorithms OSC and SSC-ASE perform well on disassortative PABMs. 
Simulations performed in this section are the same as those in Section \ref{balanced-communities} 
with the exception of the distributions from which within and between-group popularity parameters are drawn. 
Here, we draw these parameters such that the expected value is $1/3$ for the within-group popularity parameters 
and $2/3$ for the between-group popularity parameters:

\begin{itemize}
\tightlist
\item
  Number of vertices \(n = 128, 256, 512, 1024, 2048, 4096\)
\item
  Number of underlying communities \(K = 2, 3, 4\)
\item
  Mixture parameters \(\alpha_k = \frac{k^{-1}}{\sum_{\ell=1}^K \ell^{-1}}\)
  for \(k = 1, ..., K\)
\item
  Community labels
  \(z_k \stackrel{\text{iid}}{\sim} \Multinomial(\alpha_1, ..., \alpha_K)\)
\item
  Within-group popularities
  \(\lambda^{(kk)} \stackrel{\text{iid}}{\sim} \Betadist(1, 2)\)
\item
  Between-group popularities
  \(\lambda^{(k \ell)} \stackrel{\text{iid}}{\sim} \Betadist(2, 1)\) for
  \(k \neq \ell\)
\end{itemize}

OSC, SSC-ASE, and SSC-A perform similarly on disassortative PABMs compared to the simulations in Sections \ref{balanced-communities} and \ref{imbalanced-communities}. 
(Fig. \ref{fig:clust_k_disassortative} and \ref{fig:lambda_est_disassortative}).

\begin{remark}
While we call the models in this section ``disassortative'', it is our view that the assortative/disassortative distinction is not applicable to PABMs due to its flexibility. 
Unlike the SBM and DCBM, in the PABM, each vertex is free to have a higher affinity to its own community or to other communities, independent of each other. 
In other words, within the same community, $v_i$ may have a larger popularity parameter to its own community whereas $v_j$ may have a larger popularity parameter to a different community. 
Furthermore, when viewed as GRDPGs, the assortativity or disassortativity of SBMs and DCBMs affects whether $P$ is positive semidefinite, which affects ASE-based approaches to analyzing the graph \citep{rubindelanchy2017statistical}, 
but in the full rank PABM, $P$ is never positive semidefinite and will always have $K (K + 1) / 2$ positive eigenvalues and $K (K - 1) / 2$ negative eigenvalues. 
\end{remark}

\begin{figure}[H]
{\centering \includegraphics{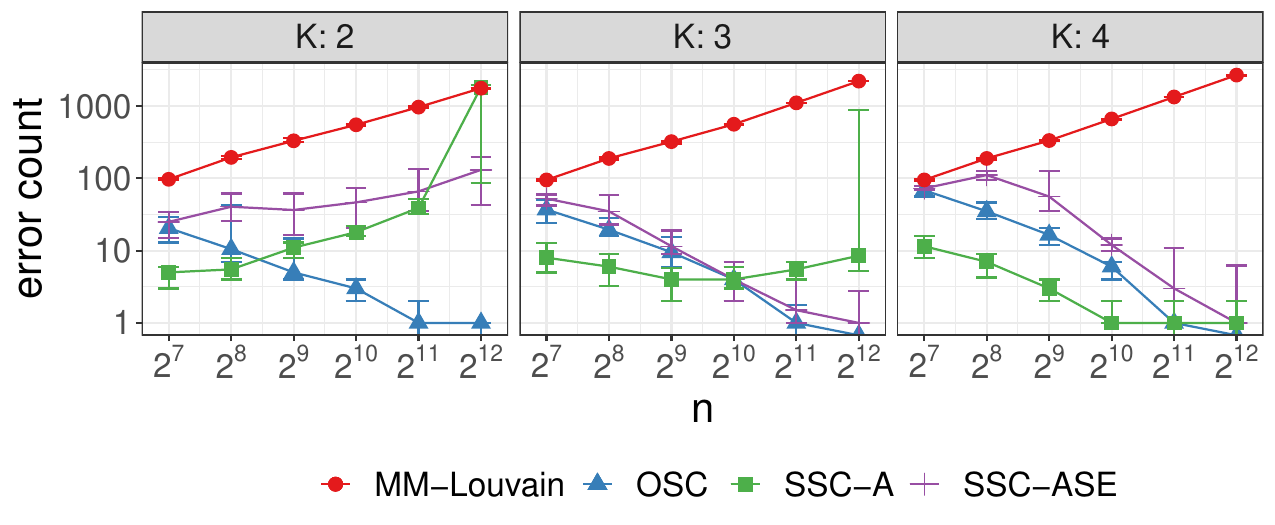}
}
\caption{Median and IQR of community detection error. Communities are approximately balanced. 
Edge probabilities were drawn to be disassortative. 
Simulations were repeated 50 times for each sample size.}\label{fig:clust_k_disassortative}
\end{figure}

\begin{figure}[H]
{\centering \includegraphics{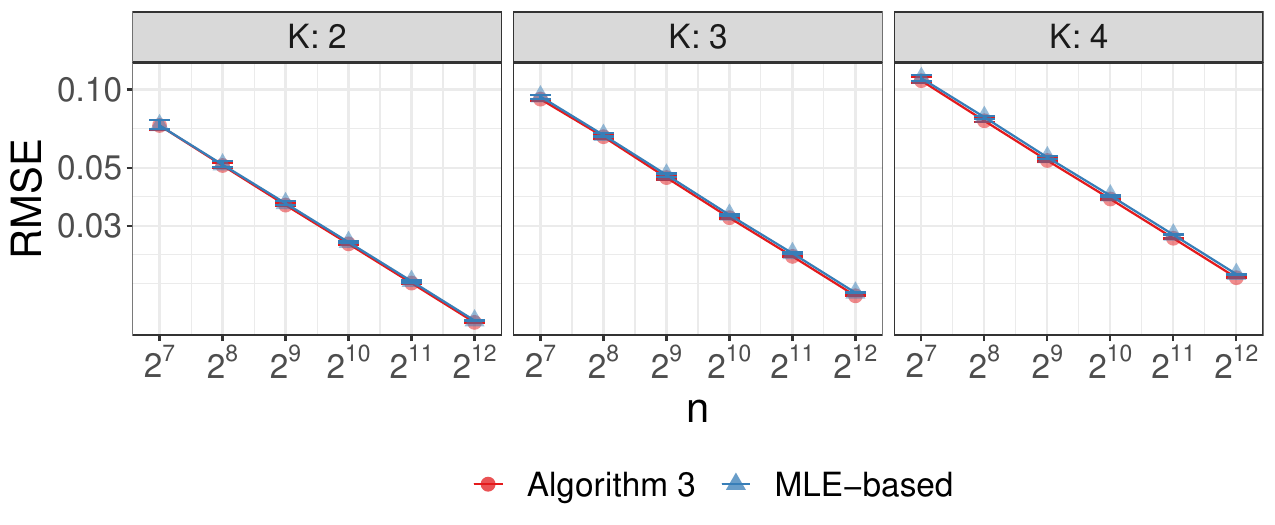}
}
\caption{Median and IQR RMSE of edge probabilities derived from the
  outputs of Algorithm \ref{alg:param_est} (red) compared against an
  MLE-based method (blue) described in
  \cite{307cbeb9b1be48299388437423d94bf1}.
  Simulations were repeated 50 times for each sample size. 
  Communities were drawn to be approximately balanced.
  Edge probabilities were drawn to be disassortative.}
\label{fig:lambda_est_disassortative}
\end{figure}

\hypertarget{real-data-examples}{%
\section{Applications}\label{real-data-examples}}

In the first example, we applied OSC (Algorithm~\ref{alg:osc}) to the Leeds Butterfly
dataset \citep{Wang_2018} consisting of visual similarity measurements
among 832 butterflies across 10 species. The graph was modified to match
the example from \citet{noroozi2019estimation}: Only the $K=4$ most
frequent species were considered, and the similarities were discretized
to \(\{0, 1\}\) via thresholding. Fig.~\ref{fig:butterfly} shows a
sorted adjacency matrix sorted by the resultant clustering.

Comparing against the ground truth species labels, \citet{noroozi2019estimation} 
report that SSC on the adjacency matrix achieves an adjusted Rand index of 73\% in their implementation, 
whereas OSC achieves 92\% and SSC on the ASE achieves 96\%. 

\begin{figure}[H]

{\centering \includegraphics{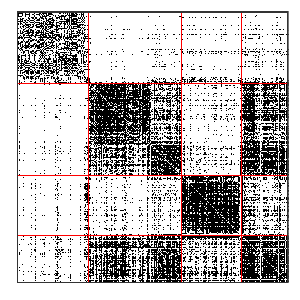}

}

\caption{Adjacency matrix of the Leeds Butterfly dataset 
after sorting by the clustering outputted by OSC.}\label{fig:butterfly}
\end{figure}

In the second example, we applied OSC to the British MPs Twitter network
\citep{greene2013producing}, the Political Blogs network
\citep{10.1145/1134271.1134277}, and the DBLP network
\citep{NIPS2009_3855, 10.1007/978-3-642-15880-3_42}. For this data
analysis, we subsetted the data as described in
\citet{307cbeb9b1be48299388437423d94bf1} for their analysis of the
same networks. Our methods slightly underperformed compared to modularity
maximization, although performance is comparable. The run time of OSC
is however much smaller than that of modularity maximization.

\begin{figure}[H]
{\centering \includegraphics{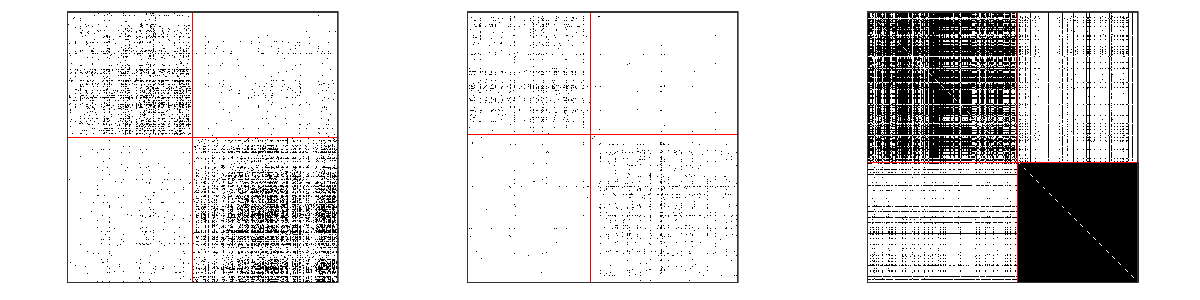}
}
\caption{Adjacency matrices of (from left to right) the British MPs, Political Blogs, and DBLP networks after sorting by the clustering outputted by OSC.}\label{fig:mp}
\end{figure}

\begin{table}
\centering
\begin{tabular}[t]{l|r|r|r}
\hline
Network & MM & SSC-ASE & OSC\\
\hline
British MPs & 0.003 & 0.012 & 0.006\\
\hline
Political Blogs & 0.050 & 0.187 & 0.062\\
\hline
DBLP & 0.028 & 0.072 & 0.059\\
\hline
\end{tabular}
\caption{\label{tab:unnamed-chunk-6}Community detection error rates on the British MPs Twitter, Political Blogs, and DBLP networks using modularity maximization, sparse subspace clustering, and OSC.}
\end{table}

In the third example (Fig.~\ref{fig:households-figure} and 
Table~\ref{tab:households-table}), 
we analyzed the Karantaka villages data studied by
\citet{DVN/U3BIHX_2013}. We chose the \texttt{visitgo}
networks from villages 12, 31, and 46 at the household level. 
In these networks, each node is a household and each edge is 
an interaction between members of pairs of households. 
The label of interest is the religious affiliation. 
The networks were truncated to religions ``1'' and ``2'', 
and vertices of degree 0 were removed. 
The villages were chosen based on there being an adequate number of nodes 
between households within each religion. 

\begin{figure}[tp]
{\centering \includegraphics{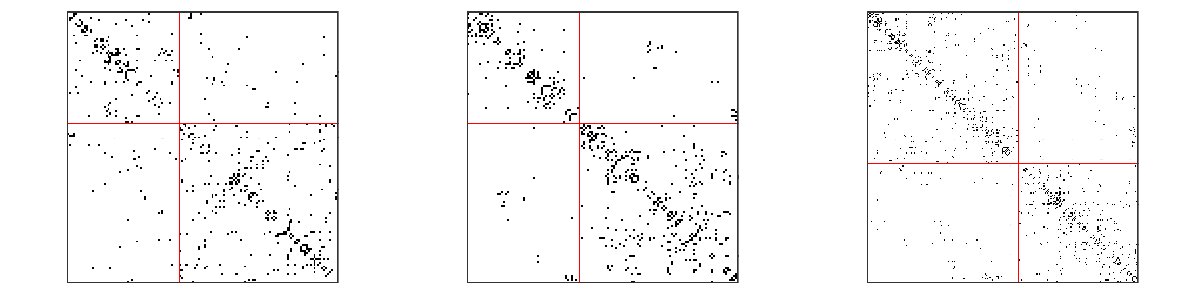}
}
\caption{Adjacency matrix of the Karnataka villages data, arranged by the clustering produced by OSC (left). The villages studied here are, from left to right, 12, 31, and 46.}\label{fig:households-figure}
\end{figure}

\begin{table}
\centering
\begin{tabular}[t]{l|r|r|r}
\hline
Network & MM & SSC-ASE & OSC\\
\hline
Village 12 & 0.270 & 0.291 & 0.227\\
\hline
Village 31 & 0.125 & 0.059 & 0.051\\
\hline
Village 46 & 0.052 & 0.069 & 0.056\\
\hline
\end{tabular}
\caption{\label{tab:households-table}Community detection error rates for identifying household religion.}
\end{table}

\hypertarget{discussion}{%
\section{Discussion}\label{discussion}}

Our central result states that the Popularity Adjusted Block Model 
is a special case of the Generalized Random Dot Product Graph. 
In particular, the PABM with $K$ communities is a GRDPG 
for which the communities are represented by 
mutually orthogonal $K$-dimensional subspaces 
of the $K^2$-dimensional latent space. 
This result extends previous results that connected 
the Stochastic Block Model and the Degree Corrected Block Model 
to Random Dot Product Graphs. 
Replacing RDPGs with GRDPGs is a critical step in this line of research, 
as a PABM is not necessarily a RDPG. 

Because all Bernoulli Graphs are GRDPGs, it should be possible to
invent and study new families of Bernoulli Graphs by characterizing
them as special cases of GRDPGs and exploiting the latent structures
that define them.  The present work illustrates the power of this
approach.  We recover the latent structure of the PABM by Adjacency
Spectral Embedding, then exploit that structure to improve statistical
inference.  Exploiting the fact that PABM communities correspond to
orthogonal subspaces, we propose Orthogonal Spectral Clustering for
community detection and demonstrate that the number of misclassified
vertices approaches zero with high probability as the size of the
graph increases.  This is a stronger result than previously proposed
algorithms \citep{307cbeb9b1be48299388437423d94bf1}, which only
guarantee that the error rate (and not count) approaches zero
asymptotically.  Parameter estimation can be performed in a similar
fashion using the ASE, for which we also prove that the per-parameter
error approaches zero asymptotically.

A secondary benefit of the GRDPG approach is that the latent structure
may be used to improve existing algorithms.  For example, one
algorithm for PABM community detection \citep{noroozi2019estimation}
relies on Sparse Subspace Clustering.  The latent structure of the
PABM provides a natural justification for SSC for the PABM and leads
to an improvement over the previous implementation.  The improved
algorithm applies SSC to the ASE, and we prove that the ASE of the
PABM obeys the Subspace Detection Property with high probability if
the graph is large.

Finally, one might well inquire what one gains and what one sacrifices
by assuming that a Bernoulli Graph is a PABM.  The GRDPG model offers
a plausible way to pursue this inquiry.  Absent a known latent
structure that can be exploited by specialized methods, the GRDPG-ASE
approach transforms the problem of network community detection to the
much-studied problem of clustering vectors in Euclidean space.
Communities of vertices are defined as clusters of latent vectors.
After ASE, a standard clustering algorithm, e.g., single linkage, is
used to infer the communities.  In future research, we intend to use
such general algorithms as baselines and measure the efficiency of the
PABM algorithms (and other specialized algorithms) by studying how
much they improve on general algorithms when the specified latent
structure obtains.

\appendix

\hypertarget{proofs}{%
\section{Proofs of Theorem~\ref{theorem4}, Theorem~\ref{theorem5}, and Theorem~\ref{theorem6}}}
Let \(V_n\) and \(\hat{V}_n\)
be the $n \times K^2$ matrices whose columns are the eigenvectors of \(P\) and \(A\) corresponding to the
$K^2$ largest eigenvalues (in modulus), respectively. 
We first state an important technical lemma for bounding the maximum
$\ell_2$ norm difference between the rows of $\hat{V}_n$ and
$V_n$. See \citet{cape_biometrika} and 
\citet[Lemma~5]{rubindelanchy2017statistical} for a proof. 

\begin{lemma}
\label{lem:technical}
Let $A \sim \mathrm{PABM}(\{\lambda^{(k \ell \ell)}\}_{K})$ be a $K$-blocks
PABM graph on $n$ vertices and let \(V\) and \(\hat{V}\)
be the $n \times K^2$ matrices whose columns are the eigenvectors of 
\(P\) and \(A\) corresponding to the
$K^2$ largest eigenvalues in modulus, respectively.
Let \(v_i^\top\) and \(\hat{v}_i^\top\) denote the $i$th 
row of \(V\) and \(\hat{V}\), respectively. 
Then there exists a constant $c > 1$ and an orthogonal matrix $W$ such
that with high probability,
$$\max_{i} \|W \hat{v}_i - v_i\|  = O\Big(\frac{\log^{c}n}{n \sqrt{\rho_n}} \Big).$$
In particular we can take $c = 1 + \epsilon$ for any $\epsilon > 0$. 
\end{lemma}

\begin{proof}[Proof of Theorem~\ref{theorem4}]
Recall the notations in Lemma~\ref{lem:technical} and note that,
under our assumption that the latent vectors $\lambda^{(k \ell)}$
are all homogeneous, we have $\max_{i} \|v_i\| =
O(n^{-1/2})$. 

Next recall Theorem~\ref{theorem3}; in particular $B_{ij} = nv_i^{\top}
v_j$. 
We therefore have
\[\begin{split}
 \max_{ij} |\hat{B}_{ij} - B_{ij}| &= \max_{ij} n |\hat{v}_i^\top \hat{v}_j -
v_i^\top v_j| \\
& \leq n \max_{ij} |\hat{v}_i^\top W W^\top \hat{v}_j -
v_i^\top v_j| \\
& \leq n \max_{i,j} \Bigl(\|W^{\top} \hat{v}_i - v_i\| \times \|\hat{v}_j\|
+ \|W^{\top} \hat{v}_j - v_j\| \times \|v_i\|\Bigr) \\
& \leq n \Bigl(\max_{ij}  \|W \hat{v}_i  - v_i \|^2 +  \|W
\hat{v}_i  - v_i \| \times \|v_j\| +  \|W \hat{v}_j  - v_j \| \times \|v_i\|\Bigr) 
\\ &
\leq n \max_{i} \|W \hat{v}_i  - v_i \|^2 + 2n \max_{i}
\|W \hat{v}_i  - v_i \| \times \max_{j} \|v_j\|
\\
& = O \Big( \frac{\log^{c}{n}}{n^{1/2} \rho_n^{1/2}} \Big)
\end{split}\]
with high probability.
Theorem~\ref{theorem4} follows from the above bound together with the
conclusion in Theorem~\ref{theorem3} that $B_{ij} = 0$ whenever vertices $i$ and $j$
belongs to different communities. 
\end{proof}

We now provide a proof of Theorem~\ref{theorem5}. Our proof is based
on verifying the sufficient conditions given in Theorem~6 of
\citet{jmlr-v28-wang13}
under which sparse subspace clustering based on solving the
optimization problem in Eq.~\eqref{eq:ssc} yields an affinity matrix
$B = |C| + |C^{\top}|$ satisfying the subspace detection property of
Definition~\ref{def:subspace_detection}. We first recall a few
definitions used in \citet{soltanolkotabi2012} and \citet{jmlr-v28-wang13}; for ease of exposition,
these definitions are stated using the notations of the current
paper and we will drop the explicit dependency on $n$ from our
eigenvectors $\hat{V}$ of $A$ and $V$ of $P$.
 % We first recall the
% notion of the inradius of a set of points. 
\begin{definition}[Inradius]
  \label{inradius}
The inradius of a convex body $\mathcal{P}$, denoted by $r(\mathcal{P})$, is
defined as the radius of the largest Euclidean ball inscribed in $\mathcal{P}$.
Let $X$ be a $n \times d$ matrix with rows $x_1, x_2, \dots,
x_n$. We then define, with a slight abuse of notation, $r(X)$ as the
inradius of the convex hull formed by $\{\pm x_1, \pm x_2, \dots, \pm x_n\}$. 
\end{definition}

\begin{definition}[Subspace incoherence]
  \label{def:subspace_incoherence}
Let $\hat{V}$ be the eigenvectors of $A$
corresponding to the $K^2$ largest eigenvalues in modulus. Let
$\hat{V}^{(k)}$ denote the matrix formed by keeping only the rows of
$\hat{V}$ corresponding to the $k^{th}$
community and let $\hat{V}^{(-k)}$ denote the matrix formed by
omitting the rows of $\hat{V}$
corresponding to the $k^{th}$ community. Let $(\hat{v}_i^{(k)})^\top$ denote
the $i$th row of $\hat{V}^{(k)}$ and $\hat{V}_{-i}^{(k)}$ be $\hat{V}^{(k)}$ with
the $i^{th}$ row omitted. Let $V$, $V^{(k)}$, $V^{(-k)}$, and
$v_i^{(k)}$ be defined similarly using the eigenvectors $V$ of
$P$. Finally let $\mathcal{S}^{(k)}$ be the vector space spanned by the
rows of $V^{(k)}$. 

Now define $\nu_{i}^{(k)}$ for $k = 1,2,\dots,K$ and $i =
1,2,\dots,n_{k}$ as the solution of the following optimization problem
$$\nu_{i}^{(k)} = \max_\eta (\hat{v}_i^{(k)})^\top \eta - \frac{1}{2
  \lambda} \eta^\top \eta, \quad \text{subject to $\|V_{-i}^{(k)}
  \eta\|_\infty \leq 1$.}$$
Given $\nu_i^{(k)}$, let $\mathbb{P}_{\mathcal{S}^{(k)}}(\nu_i^{(k)})$
be the vector in $\mathbb{R}^{K^2}$ corresponding to the orthogonal projection of $\nu_i^{(k)}$ onto the vector space
$\mathcal{S}^{(k)}$ and define the projected dual direction $w_{i}^{(k)}$
as
$$w_i^{(k)} =
\frac{\mathbb{P}_{\mathcal{S}^{(k)}}(\nu_i^{(k)})}{\|\mathbb{P}_{\mathcal{S}^{(k)}}(\nu_i^{(k)})\|}.$$
Now let $W^{(k)} = \bigl[ w_1^{(k)} \mid \cdots \mid w_{n_k}^{(k)} \bigr]^\top$
and define the subspace incoherence for $\hat{V}^{(k)}$ by
$$\mu^{(k)} = \mu(\hat{V}^{(k)}) = \max\limits_{v \in V^{(-k)}} \|W^{(k)} v\|_\infty.$$
\end{definition}

With the above definitions in place, we are now ready to state our
proof of Theorem~\ref{theorem5}.

\begin{proof}[Proof of Theorem~\ref{theorem5}]
For a given $k = 1,2\dots,K$, let $r^{(k)} = \min_{i}r(V_{-i}^{(k)})$ be inradius of the convex hull formed by
the rows of $V_{-i}^{(k)}$ and let $r_* = \min_{k} r^{(k)}$. Then Theorem~6 in
\citet{jmlr-v28-wang13} states that there exists a $\lambda > 0$
such that $\sqrt{n} \hat{V}$ satisfies
the subspace detection property in
Definition~\ref{def:subspace_detection} whenever
the following two conditions are satisfied
\begin{gather}
  \label{eq:cond1}
  \mu^{(k)} < r^{(k)} \quad \text{for all $k = 1,2,\dots,K$}, \\
  \label{eq:cond2}
  \max_{i} \|W \hat{v}_{i} - v_{i}\| \leq \min_{k} \frac{r_*(r^{(k)} -
    \mu^{(k)})}{2 + 7 r^{(k)}}.
\end{gather}
We now verify that for sufficiently large n, Eq.~\eqref{eq:cond1} and Eq.~\eqref{eq:cond2}
holds with high probability.

{\bf Verifying Eq.~\eqref{eq:cond1}}. If $n$ is sufficiently large then
there are enough vertices in each community $k$ so that
$\mathrm{span}(V_{-i}^{(k)}) = \mathcal{S}^{(k)}$ for all $i$ and hence
\(r^{(k)} = \min_{i} r(V_{-i}^{(k)}) > 0\) for
all $k = 1,2,\dots,K$. 

Next, by Theorem \ref{theorem3} we have that the subspaces
$\{\mathcal{S}^{(1)}, \dots, \mathcal{S}^{(K)}\}$
are mutually orthogonal, i.e., $v^{\top} w = 0$ for all $v \in
\mathcal{S}^{(k)}$ and $w \in \mathcal{S}^{(\ell)}$ with $k \not =
\ell$. Now let $z \in \mathbb{R}^{K^2}$ be arbitrary and let
$\tilde{z} = \mathbb{P}_{\mathcal{S}^{(k)}} z$ be the projection of
$z$ onto $\mathcal{S}^{(k)}$. We then have $v^{\top} \tilde{z} =
0$ for all $v \in V^{(-k)}$. Because $z$ is arbitrary, this implies 
$\|W^{(k)} v\|_{\infty} = 0$ for all $v
\in V^{(-k)}$ and hence $\mu^{(k)} = 0$ for all $k
=1,2,\dots,K$. Therefore $\mu^{(k)} < r^{(k)}$ for all $k =
1=2,\dots,K$ as desired.

{\bf Verifying Eq.~\eqref{eq:cond2}}.
Let $\delta = \max_{i} \sqrt{n} \|W
\hat{v}_{i} - v_{i}\|$. Then from Lemma~\ref{lem:technical}, we have
\(\delta \stackrel{a.s.}{\to} 0\) and hence
$$\delta < \min_{k} \frac{r_* (r^{(k)} - \mu^{(k)})}{2 + 7 r^{(k)}}$$
asymptotically almost surely. 

In summary $\sqrt{n} \hat{V}$ satisfies the subspace detection property
with probability converging to $1$ as \(n \to \infty\).
\end{proof}

\begin{remark}
Theorem 6 of \citet{jmlr-v28-wang13} assumes that each row $v_i$
of $V$ has unit norm, i.e., $\|v_{i}\| = 1$  for all $i$. This
assumption has the effect of scaling the $r^{(k)}$ so that $r^{(k)}
\leq 1$ for all $k = 1,2,\dots,K$. We emphasize that this assumption
has no effect on the proof of Theorem \ref{theorem5}. Indeed,
because $\mu^{(k)} = 0$ for all $k$, as long as the rows of $V^{(k)}$
spans the subspace $\mathcal{S}^{(k)}$, then $a r^{(k)} > \mu^{(k)}$ 
for any scalar $a > 0$. 
\end{remark}

\begin{proof}[Proof of Theorem \ref{theorem6}] Let \(P\) be
organized by community such that \(P^{(k \ell)}\) denote the $n_k \times
n_{\ell}$ matrix obtained by keeping only the rows of $P$
corresponding to vertices in community $k$ and the columns of $P$
corresponding to vertices in community $\ell$. We define $A^{(k
  \ell)}$ analogously. Recall that $P^{(k \ell)} = \lambda^{(k \ell)} (\lambda^{(\ell k)})^{\top}$ for all $k, \ell$. We now consider estimation of $P^{(k \ell)}$
for the cases when $k = \ell$ versus when $k \not = \ell$.

\emph{Case \(k = l\)}. Let $P^{(kk)} =
\sigma_{kk}^2 u^{(kk)} (u^{(kk)})^\top$ be the singular value
decomposition of $P^{(kk)}$. We can then define
$\tilde{\lambda}^{(kk)} = \sigma_{kk} u^{(kk)}$. 
Now let $\hat{U}^{(kk)} \hat{\Sigma}^{(kk)} (\hat{U}^{(kk)})^\top$ be the
singular value decomposition of \(A^{(kk)}\), and let
$\hat{\sigma}_{kk}^2 \hat{u}^{(kk)} (\hat{u}^{(kk)})^\top$ be the
best rank-one approximation of $A^{(kk)}$. Define
\(\hat{\lambda}^{(kk)} = \hat{\sigma}_{kk} \hat{u}^{(kk)}\). Then
\(\hat{\lambda}^{(kk)}\) is the adjacency spectral embedding approximation of \(\lambda^{(kk)}\)
and by Theorem 5 of \citet{rubindelanchy2017statistical}, we have
%adjacency spectral embedding \(\hat{\lambda}^{(kk)}\) approximates
$$\|\hat{\lambda}^{(kk)} - \lambda^{(kk)}\|_{\infty} = O\Bigl(\frac{\log n_k}{\sqrt{n_k}}\Bigr)$$
with high probability. 
Here $\|\cdot\|_{\infty}$ denote the $\ell_{\infty}$ norm of a vector.

\emph{Case \(k \neq l\)}. % \(P^{(kl)}\) and \(A^{(kl)}\) represent edge
% probabilities and edges between communities \(k\) and \(l\). Note that
% \(P^{(kl)} = (P^{(lk)})^\top\).\\
% By definition, \(P^{(kl)} = \lambda^{(kl)} (\lambda^{(lk)})^\top\). 
Let \(P^{(k \ell)} = \sigma_{k \ell}^2 u^{(k \ell)} (v^{(k \ell)})^\top\) and 
\(P^{(\ell k)} = \sigma_{kl}^2 u^{(\ell k)} (v^{(\ell k)})^\top\)  be the singular
value decompositions and note that \(\sigma_{k \ell} = \sigma_{\ell
  k}\), \(u^{(k \ell)} = v^{(\ell k)}\), and
\(v^{(k \ell)} = u^{(\ell k)}\). 
Now define \(\lambda^{(k \ell)} = \sigma_{k \ell} u^{(k \ell)}\) and
\(\lambda^{(\ell k)} = \sigma_{k \ell}
v^{(k \ell)}\).

Next consider the Hermitian dilation

\[M^{(k \ell)} = 2 \begin{bmatrix} 0 & P^{(k \ell)} \\ P^{(\ell k)} & 0 \end{bmatrix}\]

which is a symmetric \((n_k + n_\ell) \times (n_k + n_\ell)\) matrix. The
eigendecomposition of \(M^{(k \ell)}\) is then

\[M^{(k \ell)} = 
\begin{bmatrix} u^{(k \ell)} & -u^{(k \ell)} \\ v^{(k \ell)} & v^{(k \ell)} \end{bmatrix} \times 
\begin{bmatrix} \sigma^2_{kl} & 0 \\ 0 & -\sigma^2_{kl} \end{bmatrix} \times
\begin{bmatrix} u^{(k \ell)} & -u^{(k \ell)} \\ v^{(k \ell)} & v^{(k \ell)} \end{bmatrix}^\top\]
Thus treating \(M^{(k \ell)}\) as the edge probability matrix of a GRDPG, we
have latent positions in \(\mathbb{R}^2\) given by the $(n_k + n_{\ell}) \times 2$ matrix
\[\Lambda^{(k \ell)} = \begin{bmatrix} 
  \sigma_{k \ell} u^{(k \ell)} & \sigma_{k \ell} u^{(k \ell)} \\ 
  \sigma_{k \ell} v^{(k \ell)} & -\sigma_{k \ell} v^{(k \ell)} 
\end{bmatrix} = 
\begin{bmatrix} 
  \lambda^{(k \ell)} & \lambda^{(k \ell)} \\ 
  \lambda^{( \ell k)} & -\lambda^{( \ell k)} 
\end{bmatrix}.\]
Now consider
\[\hat{M}^{(k \ell)} = \begin{bmatrix} 0 & A^{(k \ell)} \\ A^{(\ell k)} & 0 \end{bmatrix}\]

We can then view \(\hat{M}^{(k \ell)}\) as an adjacency matrix drawn from
the edge probabilities matrix \(M^{(k \ell)}\). Now suppose that the adjacency spectral
embedding of $\hat{M}^{(k \ell)}$ is represented as the $(n_k +
n_{\ell}) \times 2$ matrix
\[\hat{\Lambda}^{(k \ell)} = \begin{bmatrix} 
  \hat{\lambda}^{(k \ell)} & \hat{\lambda}^{(k \ell)} \\ 
  \hat{\lambda}^{(\ell k)} & -\hat{\lambda}^{(\ell k)} 
\end{bmatrix}\]
where each \(\hat{\lambda}^{(k \ell)}\) is defined as in Algorithm 3. Then
by Theorem 5 of \citet{rubindelanchy2017statistical}, there
exists an indefinite orthogonal transformation $W^{*}$ such that,  with
high probability,
$$\max_{i} |W^{*} \hat{\Lambda}_{i}^{(k \ell)} - \Lambda_{i}^{(k \ell)} \| =
O\Bigl(\frac{\log (n_k + n_{\ell})}{\sqrt{n_k + n_{\ell}}}\Bigr)$$
with high probability. Here $\Lambda_{i}^{(k \ell)}$ and
$\hat{\Lambda}_i^{(k \ell)}$ denote the $i$th rows of $\Lambda^{(k
  \ell)}$ and $\hat{\Lambda}^{(k \ell)}$, respectively. 

Furthermore, by looking at the proof of Theorem~5 in
\citep{rubindelanchy2017statistical}, we see that $W^{*}$ is also
blocks diagonal with $2$ blocks where the positive eigenvalues of $M^{(k \ell)}$
forming a block and the negative eigenvalues of $M^{(k \ell)}$ forming
the remaining block. 
Because $M^{(k \ell)}$ has one positive eigenvalue and one negative
eigenvalue, we see that $W^{*}$ is necessarily of the form $W^{*}
= \Bigl[\begin{smallmatrix} 1 & 0 \\ 0 & - 1\end{smallmatrix}\Bigr]$
Using this form for $W^{*}$, we obtain
$$\max\{\|\hat{\lambda}^{(k \ell)} - \lambda^{(k \ell)}\|_{\infty},
\|\hat{\lambda}^{(\ell k)} - \lambda^{(\ell k)}\|_{\infty}\}  =
O\Bigl(\frac{\log(n_k + n_{\ell})}{\sqrt{n_k + n_{\ell}}}\Bigr)$$
with high probability. Combining this bound with the bound for
$\|\hat{\lambda}^{(kk)} - \lambda^{(kk)}\|_{\infty}$ given above
yields Eq.~\eqref{eq:thm6} in Theorem~\ref{theorem6}. 
\end{proof}

\section*{Acknowledgements}

This work was partially supported by the Naval Engineering Education Consortium (NEEC), Office of Naval Research (ONR) Award Number N00174-19-1-0011.

\bigskip
\begin{center}
{\large\bf SUPPLEMENTAL MATERIALS}
\end{center}

\begin{description}

\item[Source files:] The source files used to compile this document can be found in summary.zip.

\item[Data:] The data used in section 5 can be found in data.zip.

\item[Code:] The code for simulations and data anlyses can be found in code-and-results.zip. 
This file also includes the simulation results in csv format (please note that the simulations can take a long time to complete). 
The code for data analyses requires the files found in data.zip. 
All of these files can also be found at\\
\url{https://github.com/johneverettkoo/pabm-grdpg}.

\item[Sparsity:] Additional supplemental materials regarding simulations for the effect of the sparsity parameter can be found in sparsity-sim.zip. 

\end{description}

\renewcommand\refname{References}
% \bibliographystyle{plain}
% \bibliography{misc.bib}
\printbibliography

\end{document}

%%% Local Variables:
%%% mode: latex
%%% TeX-master: t
%%% End: